\documentclass[journal,twoside,web]{ieeecolor}

\usepackage{mathtools}
\usepackage{generic}
\usepackage{amsmath,amssymb,bbm}
\let\labelindent\relax
\usepackage{enumitem}
\usepackage{cancel}
\usepackage{graphicx}
\usepackage{pifont}
\usepackage{xspace}
\usepackage{xcolor}
\usepackage{makecell}
\usepackage{hyperref}

\newtheorem{assumption}{Assumption}
\newtheorem{lemma}{Lemma}

\newtheorem{remark}{Remark}
\newtheorem{corollary}{Corollary}
\newtheorem{theorem}{Theorem}

\allowdisplaybreaks 

\usepackage{etoolbox}
\makeatletter
\@ifundefined{color@begingroup}%
  {\let\color@begingroup\relax
   \let\color@endgroup\relax}{}%
\def\fix@ieeecolor@hbox#1{%
  \hbox{\color@begingroup#1\color@endgroup}}
\patchcmd\@makecaption{\hbox}{\fix@ieeecolor@hbox}{}{\FAILED}
\patchcmd\@makecaption{\hbox}{\fix@ieeecolor@hbox}{}{\FAILED}
\usepackage{algorithm}
\usepackage{algpseudocode}
\algrenewcommand\algorithmicrequire{\textbf{Input:}}
\algrenewcommand\algorithmicensure{\textbf{Output:}}
\newcommand{\CommentState}[1]{\Statex\hspace{\algorithmicindent}{\color{blue}// #1}}

\newcommand{\algsgd}{LT-ADMM\xspace}
\newcommand{\algsaga}{LT-ADMM-VR\xspace}
\newcommand{\algsagatwo}{LT-ADMM-VR v2\xspace}

\DeclareMathOperator*{\argmin}{arg\,min}

\newcommand{\N}{\mathbb{N}}
\newcommand{\R}{\mathbb{R}}

\newcommand{\norm}[1]{\left\lVert#1\right\rVert}

\newcommand{\cmark}{{\color{green!80!black}\ding{51}\xspace}}
\newcommand{\xmark}{{\color{red}\ding{55}\xspace}}

\newcommand{\sgd}[1]{\colorbox{green!30!white}{#1}}
\newcommand{\saga}[1]{\colorbox{blue!30!white}{#1}}

\newcommand{\tgrad}{t_\mathrm{G}\xspace}
\newcommand{\tcomm}{t_\mathrm{C}\xspace}

\title{Communication-Efficient\\Stochastic Distributed Learning}

\author{%
	Xiaoxing~Ren$^{1}$, Nicola~Bastianello$^{2\star}$, Karl~H.~Johansson$^{2}$, Thomas~Parisini$^{3,4,5}$
	\thanks{The work of X.R. and T.P. was supported by European Union's Horizon 2020 Research and Innovation programme under grant agreement no. 739551 (KIOS CoE).}
\thanks{The work of N.B. and K.H.J. was partially supported by the European Union’s Horizon Research and Innovation Actions programme under grant agreement No. 101070162, and partially by Swedish Research Council Distinguished Professor Grant 2017-01078 Knut and Alice Wallenberg Foundation Wallenberg Scholar Grant.}
\thanks{$^{1}$School of Civil and Environmental
Engineering, Systems Engineering Field, Cornell University, Ithaca, NY,
USA. }
	\thanks{$^{2}$School of Electrical Engineering and Computer Science, and Digital Futures, KTH Royal Institute of Technology, Stockholm, Sweden.}
 	\thanks{$^{3}$Department of Electrical and Electronic Engineering, Imperial College London, London, United Kingdom.}%
	\thanks{$^{4}$Department of Electronic Systems, Aalborg University, Denmark.}%
	\thanks{$^{5}$Department of Engineering and Architecture, University of Trieste, Trieste, Italy.}%
	\thanks{$^{\star}$Corresponding author. Email: {\tt\footnotesize nicolba@kth.se}.}%
}

\begin{document}

\maketitle

\begin{abstract}
We address distributed learning problems, both nonconvex and convex, over undirected networks. In particular, we design a novel algorithm based on the distributed \textit{Alternating Direction Method of Multipliers} (ADMM) to address the challenges of high communication costs, and large datasets.
Our design tackles these challenges i) by enabling the agents to perform multiple local training steps between each round of communications; and ii) by allowing the agents to employ stochastic gradients while carrying out local computations.
We show that the proposed algorithm converges to a neighborhood of a stationary point, for nonconvex problems, and of an optimal point, for convex problems.
We also propose a variant of the algorithm to incorporate variance reduction thus achieving exact convergence. We show that the resulting algorithm indeed converges to a stationary (or optimal) point, and moreover that local training accelerates convergence.
We thoroughly compare the proposed algorithms with the state of the art, both theoretically and through numerical results.
\end{abstract}


\section{Introduction}
Recent technological advances have enabled the widespread adoption of devices with computational and communication capabilities in many fields, for instance, power grids \cite{molzahn_survey_2017}, robotics \cite{shorinwa2024distributed_tutorial,shorinwa2024distributed}, transportation networks \cite{mohebifard2018distributed}, and sensor networks~\cite{nedic_distributed_2018}.
These devices connect with each other, forming multi-agent systems that cooperate to collect and process data~\cite{gafni_federated_2022}. As a result, there is a growing need for algorithms that enable efficient and accurate cooperative learning.

In specific terms, the objective in distributed learning is to train a model (\textit{e.g.}, a neural network) with parameters $x \in \R^n$ {\it cooperatively} across a network of $N$ agents. {Each agent $i$}  has access to a local dataset which defines the local cost as
\begin{equation}\label{eq:erm-cost}
    f_{i}({x}) = \frac{1}{m_i} \sum_{h=1}^{m_i} f_{i,h}(x) \, ,
\end{equation}
with $f_{i,h}: \mathbb{R}^{n} \rightarrow \mathbb{R}$ being the loss function associated to data point $h \in \{ 1, \ldots, m_i \}$.
Thus, the goal is for the agents to solve the following constrained problem \cite{bottou2018optimization,alghunaim2022unified}:
\begin{equation}\label{eq:optimization-problem}
     {
     \min_{\substack{x_i \in\mathbb{R}^{n}\\i ={1,...N}}} \ \ \frac{1}{N} \sum_{i=1}^{N}f_{i}({x}_i) \quad \text{s.t.} \ \ x_1 = x_2 = \cdots = x_N \, ,
     }
\end{equation}
where the objective is the sum of local costs~\eqref{eq:erm-cost} to pool together the agents' data. Moreover, each agent is assigned a set $x_i$ of local model parameters, and the consensus constraints $x_1 = x_2 = \ldots = x_N$ ensure that the agents will asymptotically agree on a shared trained model.

\begin{table*}[!t]
\centering
\caption{Comparison with the state of the art in stochastic distributed optimization.}
\label{tab:comparison}
    \begin{tabular}{cccccccc}
    \hline
    Algorithm [Ref.] & \thead{variance\\reduction} & \thead{grad. steps\\$\div$ comm.} & \thead{\# stored \\ variables$^\dagger$} & \thead{comm.\\size$^\ddagger$} & \thead{\# $\nabla f_{i,j}$ evaluations\\per iteration} & assumpt.$^\star$ & \thead{convergence} \\
    \hline
    
    K-GT \cite{liu_decentralized_2023} & \xmark & $\tau \div 1$ & $2$ & $2 |\mathcal{N}_i|$ & $1$ & n.c. & sub-linear, $\propto \sigma^2$ \\

    LED \cite{alghunaim_local_2023} & \xmark & $\tau \div 1$ & $2$ & $|\mathcal{N}_i|$ & $1$ & \thead{n.c.\\s.c.} & \thead{sub-linear, $\propto \sigma^2$\\linear, $\propto \sigma^2$} \\

    RandCom \cite{guo_randcom_2023} & \xmark & \thead{$\left\lceil \frac{1}{p} \right\rceil \div 1$\\(in mean)} & $2$ & $|\mathcal{N}_i|$ & $1$ & \thead{n.c.\\s.c.} & \thead{sub-linear, $\propto \sigma^2$\\linear, $\propto \sigma^2$} \\

    \thead{VR-EXTRA/DIGing \cite{li_variance_2022},\\GT-VR \cite{jiang_distributed_2023}} & \cmark & $1 \div 1$ & $3$ & $2 |\mathcal{N}_i|$ & $|\mathcal{B}|$, $m_i$ every $\left\lceil \frac{1}{p} \right\rceil$ & \thead{n.c.\\s.c.} & \thead{sub-linear, $\to 0$\\linear, $\to 0$} \\

    GT-SAGA \cite{xin_variance-reduced_2020,xin_fast_2022b} & \cmark & $1 \div 1$ & $3$ & $2 |\mathcal{N}_i|$ & $1$ & \thead{s.c.\\n.c.} & \thead{sub-linear, $\to 0$\\linear, $\to 0$} \\

    GT-SARAH \cite{xin_fast_2022} & \cmark & $1 \div 1$ & $3$ & $2 |\mathcal{N}_i|$ & $|\mathcal{B}|$, $m_i$ every $\tau$ & n.c. & sub-linear, $\to 0$ \\

    GT-SVRG \cite{xin_variance-reduced_2020} & \cmark & $1 \div 1$ & $3$ & $2 |\mathcal{N}_i|$ & $1$, $m_i$ every $\tau$ & s.c. & linear, $\to 0$ \\

    \hline

    \algsgd [this work] & \xmark & $\tau \div 1$ & $|\mathcal{N}_i| + 1$ & $|\mathcal{N}_i|$ & $|\mathcal{B}|$ & n.c. & sub-linear, $\propto \sigma^2$ \\
    
    \algsaga [this work] & \cmark & $\tau \div 1$ & $|\mathcal{N}_i| + 1$ & $|\mathcal{N}_i|$ & $|\mathcal{B}|$, $m_i$ every $\tau$ & n.c. & sub-linear, $\to 0$ \\
    
    \hline
    \end{tabular}
    \\\vspace{0.1cm}
    $^\dagger$ number of vectors in $\R^n$ stored by each agent between iterations (disregarding temporary variables) \\
    $^\ddagger$ number of messages sent by each agent during a communication round \\
    $^\star$ n.c. and s.c. stand for (non)convex and strongly convex
\end{table*}

To effectively tackle this problem, especially when dealing with large datasets that involve sensitive information, distributed methods have become increasingly important. These techniques offer significant robustness advantages over federated learning algorithms \cite{gafni_federated_2022}, as they do not rely on a central coordinator and thus, for example, have not a single point of failure.
In particular, both distributed gradient-based algorithms \cite{shi2015extra,nedic2017achieving,saadatniaki2020decentralized,ren2022accelerated}, and distributed Alternating Direction Method of Multipliers (ADMM) \cite{bastianello_asynchronous_2021,bastianello_robust_2024,makhdoumi2017convergence,khatana2022dc} have proven to be effective strategies for solving such problems. {
ADMM-based algorithms have demonstrated strong robustness to practical constraints (see \textit{e.g.} \cite{bastianello_robust_2024} and references therein), although this often comes at the cost of higher computational complexity compared to gradient-based methods. In this work, we propose novel ADMM-based algorithms that retain the computational efficiency characteristic of gradient methods.}

However, many learning applications face the challenges of: high communication costs, especially when training large models, and large datasets. In this paper, we jointly address these challenges with the following approach.
{First, we guarantee the communication efficiency of our algorithm by adopting the paradigm of \textit{local training}, which reduces the frequency of communications. In other terms, the agents perform multiple training steps between communication rounds.}
We tackle the second challenge by locally incorporating \textit{stochastic gradients}. The idea is to allow the agents to estimate local gradients by employing only a (random) subset of the available data, thus avoiding the computational burden of full gradient evaluations on large datasets.

Our main contributions are as follows:
\begin{itemize}
\item We propose two algorithms based on distributed ADMM, with one round of communication between multiple local update steps. The first algorithm, Local Training ADMM (\algsgd), uses stochastic gradient descent (SGD) for the local updates, while the second algorithm, LT-ADMM with Variance Reduction (\algsaga), uses a variance-reduced SGD method \cite{defazio2014saga}.

    \item  We establish the convergence properties of \algsgd for both nonconvex and convex (not strongly convex) learning problems. In particular, we show almost-sure and mean-squared convergence of \algsgd to a neighborhood of the stationary point in the nonconvex case, and to a neighborhood of an optimum in the convex case. The radius of the neighborhood depends on specific properties of the problem and on tunable parameters. We prove that the algorithm achieves a convergence rate of $\mathcal{O}(\frac{1}{K \tau})$, where $K$ is the number of iterations, and $\tau$ the number of local training steps.

    \item For \algsaga, we prove \textit{exact convergence} to a stationary point in the nonconvex case, and to an optimum in the convex case. The algorithm has a $\mathcal{O}(\frac{1}{K \tau})$ rate of convergence, which is faster than $\mathcal{O}(\frac{1}{K})$ obtained by related algorithms \cite{xin_fast_2022,xin_fast_2022b,jiang_distributed_2023}.

    \item We provide extensive numerical evaluations comparing the proposed algorithms with the state of the art. The results validate the communication efficiency of the algorithms. Indeed, \algsgd and \algsaga outperform alternative methods when communications are expensive.
\end{itemize}

\subsection{Comparison with the state of the art}\label{subsec:sota}

We compare our proposed algorithms -- \algsgd and \algsaga -- with the state of the art.
The comparison is holistically summarized in Table~\ref{tab:comparison}.

Decentralized learning algorithms, as first highlighted in the seminal paper \cite{mcmahan_communication_2017} on federated learning, face the fundamental challenge of high communication costs.
The authors of \cite{mcmahan_communication_2017} address this challenge by designing a communication-efficient algorithms which allows the agents to perform multiple local training steps before each round of communication with the coordinator.
However, the accuracy of the algorithm in \cite{mcmahan_communication_2017} degrades significantly when the agents have heterogeneous data. Since then, alternative federated learning algorithms, \textit{e.g.}, \cite{zhang_fedpd_2021,mishchenko_proxskip_2022,mitra_linear_2021,condat_tamuna_2023}, have been designed to employ local training without compromising accuracy.
The interest for communication-efficient algorithms has more recently extended to the distributed set-up, where agents rely on peer-to-peer communications rather than on a coordinator as in federated learning.
Distributed algorithms with local training have been proposed in \cite{hien_nguyen_performance_2023,liu_decentralized_2023,alghunaim_local_2023,guo_randcom_2023}. In particular, \cite{hien_nguyen_performance_2023,liu_decentralized_2023,alghunaim_local_2023} present gradient tracking methods which allow each agent to perform a fixed number of local updates between each communication round.
The algorithm in \cite{guo_randcom_2023}, which builds on \cite{mishchenko_proxskip_2022}, instead triggers communication rounds according to a given probability distribution, resulting in a time-varying number of local training steps.
Another related algorithm is that of \cite{berahas_convergence_2021}, which allows for both multiple consensus and gradient steps in each iteration. However, this algorithm requires a monotonically increasing number of communication rounds in order to guarantee exact convergence. A stochastic version of \cite{berahas_convergence_2021}  was then studied in \cite{iakovidou_snear-dgd_2023}. The algorithm has inexact gradient evaluations, but only allows for multiple consensus steps.
{An alternative approach to reducing the frequency of communications is to employ event-triggering, see \textit{e.g.} \cite{hou2024prescribed}, where messages are exchanged only when a  certain condition is met.}

When the agents employ stochastic gradients in the algorithms of \cite{liu_decentralized_2023,alghunaim_local_2023,guo_randcom_2023}, they only converge to a neighborhood of a stationary point, whose radius is proportional to the stochastic gradient variance.
Different \textit{variance reduction} techniques are available to improve the accuracy of (centralized) algorithms relying on stochastic gradients, \textit{e.g.}, \cite{johnson2013accelerating,defazio2014saga,nguyen2017sarah}.
Then, these methods have been applied to distributed optimization by combining them with widely used gradient tracking algorithms \cite{li_variance_2022,jiang_distributed_2023,xin_variance-reduced_2020,xin_fast_2022b,xin_fast_2022}.
The resulting algorithms succeed in guaranteeing exact convergence to a stationary point despite the presence of gradient noise. However, they are not communication-efficient, as they only allow one gradient update per communication round.

{We conclude by providing in Table~\ref{tab:comparison} a summary of the key features of the algorithms discussed above. This table focuses on methods that employ the mechanisms of primary interest in this work -- local training and variance reduction.}
First, we classify them based on whether they use or not variance reduction and local training. For the latter, we report the ratio of gradient steps to communication rounds that characterizes each algorithm, with a ratio of $1 \div 1$ signifying that no local training is used. Notice that \algsaga is the only algorithm to use both variance reduction and local training, while the other only use one technique.
We then compare the number of variables stored by the agents when they deploy each algorithm (disregarding temporary variables). We notice that the variable storage of \algsgd and \algsaga, differently from the alternatives, scales with the size of an agent's neighborhood; this is due to the use of distributed ADMM as the foundation of our proposed algorithms \cite{bastianello_asynchronous_2021}.
We see that \cite{alghunaim_local_2023,guo_randcom_2023}, \algsgd, and \algsaga require one communication per neighbor, while the other methods require two communications per neighbor.
We also compare the algorithms by the computational complexity of the gradient estimators they employ, namely, the number of component gradient evaluations needed per local training iteration. The algorithms of \cite{liu_decentralized_2023,alghunaim_local_2023,guo_randcom_2023,xin_variance-reduced_2020} use a single data point to estimate the gradient, while \cite{li_variance_2022,xin_fast_2022}, \algsgd, \algsaga can apply mini-batch estimators that use a subset $\mathcal{B}$ of the local data points. The use of mini-batches yields more precise gradient estimates and increased flexibility. However, we remark that the gradient estimators used in \cite{li_variance_2022,xin_fast_2022,xin_variance-reduced_2020}, \algsgd, \algsaga require a registry of component gradient evaluations, which needs to be refreshed entirely at fixed intervals. This coincides with the evaluation of a full gradient, and thus requires $m_i$ component gradient evaluations.
Finally, we compare the algorithms' convergence. We notice that all algorithms, except for \cite{xin_variance-reduced_2020}, provide (sub-linear) convergence guarantees for convex and nonconvex problems. Additionally, some works show linear convergence for strongly convex problems. We further distinguish between algorithms which achieve exact convergence due to the use of variance reduction, or inexact convergence with an error proportional to the stochastic gradient variance ($\propto \sigma^2$).

\smallskip

\subsubsection*{Outline} The outline of the paper is as follows. Section~\ref{sec:problem-design} formulates the problem at hand, and presents the proposed algorithms design. Section~\ref{sec:convergence} analyzes their convergence, and discusses the results. Section~\ref{sec:numerical} reports and discusses numerical results comparing the proposed algorithms with the state of the art. Section~\ref{sec:conclusions} presents some concluding remarks.

\subsubsection*{Notation} $\nabla f$ denotes the gradient of a differentiable function $f$.
Given a matrix $A\in \R^{n\times n}$, $\lambda_{\min}(A)$ and  $\lambda_{\max}(A)$ denotes the smallest and largest eigenvalue of $A$, respectively. $A >0$ represents that matrix $A$ is positive definite.
With $n\in \mathbb{N}$, we let $\mathbf{1}_n \in\mathbb{R}^n$ be the vector with all elements equal to $1$, $\mathbf{I} \in \R^{n \times n}$ the identity matrix and $\mathbf{0} \in \R^{n \times n}$ the zero matrix.
$\langle x, y \rangle =  \sum_{h=1}^n x_h y_h$ represents the standard inner product of two vectors $x, y \in \R^n$.
$\| \cdot \|$ denotes the Euclidean norm of a vector and the matrix-induced 2-norm of a matrix.
The proximal of a cost $f$, with penalty $\rho > 0$, is defined as $\operatorname{prox}_{f}^\rho(y) = \argmin_{y \in \R^n} \left\{ f(y) + \frac{1}{2\rho} \norm{y - x}^2 \right\}$.

\section{Problem Formulation and Algorithm Design}\label{sec:problem-design}
In this section, we formulate the problem at hand and present our proposed algorithms.

\subsection{Problem formulation}
We target the solution of~\eqref{eq:optimization-problem} over a (undirected) graph $\mathcal{G} = (\mathcal{V},\mathcal{E})$, where $\mathcal{V} = \lbrace 1,...,N \rbrace$ is the set of $N$ agents, and $\mathcal{E} \subset \mathcal{V} \times \mathcal{V}$ is the set of edges $(i,j)$, $i, j \in \mathcal{V}$. In particular, we assume that the local costs $f_i: \mathbb{R}^{n} \rightarrow \mathbb{R}$ are in the \textit{empirical risk minimization} form~\eqref{eq:erm-cost}.
We make the following assumptions for~\eqref{eq:optimization-problem}, {which are commonly used to support the convergence analysis of distributed learning algorithms (see \textit{e.g.} \cite{liu_decentralized_2023,alghunaim_local_2023,guo_randcom_2023,li_variance_2022,xin_fast_2022b,xin_fast_2022}).}

\begin{assumption}\label{as:graph}
$\mathit{\mathcal{G}} = (\mathcal{V},\mathcal{E})$ is a connected, undirected graph.
\end{assumption}

\begin{assumption}\label{as:local-costs}
The cost function $f_i$ of each agent $i \in \mathcal{V}$ is $L$-smooth. That is, there exists $l > 0$ such that 
	$
	\Vert \nabla f_{i}(x)-\nabla f_{i}(y)\Vert  \leq L \Vert x-y \Vert 
	$,
$\forall x, y \in \mathbb{R}^n$.
Moreover, $f_i$ is proper: $f_i(x) > -\infty$, $\forall x \in \mathbb{R}^{n}$.
\end{assumption}

When, in the following, we specialize our results to convex scenarios, we resort to the additional assumption below.

\begin{assumption} \label{as:convex}
Each function $f_i$, $i \in \mathcal{V}$, is convex.
\end{assumption}

\subsection{Algorithm design}
We start our design from the distributed ADMM, characterized by the updates\footnote{We remark that, more precisely, \eqref{eq:admm} corresponds to the algorithm in \cite{bastianello_asynchronous_2021} with relaxation parameter $\alpha = 1/2$.} \cite{bastianello_asynchronous_2021}:
\begin{subequations}\label{eq:admm}
\begin{align}
x_{i, k+1} &= \operatorname{prox}_{f_i}^{1 / \rho\left|\mathcal{N}_i\right|}\Bigg(\frac{1}{\rho \left|\mathcal{N}_i\right|} \sum_{j \in \mathcal{N}_i} z_{i j, k} \Bigg) \, , \label{eq:admm-x} \\
z_{i j, k+1} &= \frac{1}{2} \left(z_{i j, k} - z_{j i, k}+2 \rho x_{j, k+1}\right) \, , \label{eq:admm-z}
\end{align}
\end{subequations}
where $\mathcal{N}_i = \{ j \in \mathcal{V} \ | \ (i, j) \in \mathcal{E} \}$ denotes the neighbors of agent $i$, $\rho>0$ is a penalty parameter, and $z_{ij} \in \R^n$ are auxiliary variables, one for each neighbor of agent $i$.
This algorithm converges in a wide range of scenarios, and, differently from most gradient tracking approaches, shows robustness to many challenges (asynchrony, limited communications, \textit{etc}) \cite{bastianello_asynchronous_2021,bastianello_robust_2024}.
However, the drawback of~\eqref{eq:admm} is that the agents need to solve an optimization problem to update $x_i$, which in general does not have a closed-form solution. Therefore, in practice, the agents need to compute an approximate update~\eqref{eq:admm-x}, which can lead to inexact convergence \cite{bastianello_robust_2024}.

In this paper, we modify~\eqref{eq:admm} to use approximate local updates, \textit{while ensuring that this choice does not compromise exact convergence}. In particular, we allow the agents to use $\tau \in \N$ iterations of a gradient-based solver to approximate~\eqref{eq:admm-x}, which yields the update:
\begin{align} 
 &\phi_{i,k}^0 =x_{i, k} \, , \nonumber
 \\&
\phi_{i,k}^{t+1}  =  \phi_i^{t}-\Bigg( \gamma g_i(\phi_{i,k}^{t}) + \beta \Big(\rho\left|\mathcal{N}_i\right| x_{i, k} - \sum_{j \in \mathcal{N}_i} z_{i j, k} \Big) \Bigg), \nonumber 
\\& \qquad t =0, \ldots, \tau-1 \, , \nonumber \\&
x_{i, k+1}  =\phi_{i,k}^\tau  \, , \label{eq:local-training}
\end{align}
where $\gamma$, $\beta$ are the positive step-sizes, and $g_i(\phi_{i,k}^{\ell})$ is an estimate of the gradient $\nabla f_i$. {Notice that for efficiency's sake we ``freeze'' the penalty term $\rho\left|\mathcal{N}_i\right| x_{i, k} - \sum_{j \in \mathcal{N}_i} z_{i j, k}$, and for flexibility we multiply gradient estimate and penalty term by two different step-sizes.}
The resulting algorithm is a distributed gradient method, with the difference that each communication round~\eqref{eq:admm-z} is preceded by $\tau > 1$ local gradient evaluations. This is an application of the \textit{local training} paradigm \cite{alghunaim_local_2023}.
We remark that the convergence of the proposed algorithm rests on the initialization $\phi_{i,k}^0 =x_{i, k}$, which enacts a feedback loop on the local training. In general, without this initialization, exact convergence cannot be achieved~\cite{bastianello_robust_2024}.

The local training~\eqref{eq:local-training} requires a local gradient evaluation or at least its estimate. In the following, we introduce two different estimator options. Notice that the gradient of the penalty term, $\rho\left|\mathcal{N}_i\right| x_{i,k}-\sum_{j \in \mathcal{N}_i} z_{i j, k}$, is exactly known (and frozen) and does not need an estimator.
The most straightforward idea is to simply employ a local gradient $g_i(\phi) = \nabla f_i(\phi)$. However, in learning applications, the agents may store large datasets ($m_i \gg 1$). Therefore, computing $\nabla f_i(\phi)$ becomes computationally expensive.
To remedy this, the agents can instead use \textit{stochastic gradients}, choosing
\begin{equation}\label{eq:sgd-gradient}
    g_i(\phi) = \frac{1}{|\mathcal{B}_i|} \sum_{h \in \mathcal{B}_i} \nabla f_{i,h}(\phi) \, ,
\end{equation}
where $\mathcal{B}_i$ are randomly drawn indices from $\{ 1, \ldots, m_i\}$, with $|\mathcal{B}_i| < m_i$.
While reducing the computational complexity of the local training iterations, the use of stochastic gradients results in inexact convergence.
The idea, therefore, is to employ a gradient estimator based on a \textit{variance reduction} scheme. {In particular, we adopt the scheme proposed in \cite{defazio2014saga}, characterized by the following procedure.}
Each agent maintains a table of component gradients $\{ \nabla f_{i,h}(r_{i, h, k}^t) \}, h = 1, \ldots ,m_i$, where $r_{i, h, k}^t$ is the most recent iterate at which the component gradient was evaluated. This table is reset at the beginning of every new local training (that is, for any $k \in \N$ when $t = 0$).
Using the table, the agents then estimate their local gradients as
\begin{equation}\label{eq:saga-gradient}
\begin{split}
g_i\left(\phi_{i,k}^t\right) &= \frac{1}{|\mathcal{B}_i|} \sum_{h \in \mathcal{B}_i} \left( \nabla f_{i, h}\left( \phi_{i, k}^t \right) - \nabla f_{i, h}\left(r_{i, h, k}^{t} \right) \right) \\ &+\frac{1}{m_i} \sum_{h=1}^{m_i} \nabla f_{i, h}(r_{i, h, k}^{t}).
\end{split}
\end{equation}
The gradient estimate is then used to update $\phi_{i,k}^{t+1}$ according to~\eqref{eq:local-training}; afterwards, the agents update their local memory by setting $r_{i,h,k}^{t+1} = \phi_{i,k}^{t+1}$ if $h \in \mathcal{B}_i$, and $r_{i,h,k}^{t+1} = r_{i,h,k}^{t}$ otherwise.
Notice that this update requires a full gradient evaluation at the beginning of each local training, to populate the memory with $\{ \nabla f_{i,h}(r_{i, h, k}^0) = \nabla f_{i,h}(\phi_{i,k}^0) \}, h = 1, \ldots ,m_i$. In the following steps ($t > 0$), each agent only computes $|\mathcal{B}_i|$ component gradients.

\smallskip

{Selecting the stochastic gradient estimator~\eqref{eq:sgd-gradient} yields the proposed algorithm \algsgd, while selecting the variance reduction scheme~\eqref{eq:saga-gradient} yields the proposed algorithm \algsaga. The two methods are reported in Algorithm~\ref{alg:lt-saga-admm}.}
\begin{algorithm}[!ht]
\caption{\sgd{\algsgd} and \saga{\algsaga}}
\label{alg:lt-saga-admm}
\begin{algorithmic}[1]
	\Require For each node $i$, initialize $x_{i,0}=z_{ij, 0}$, $j \in \mathcal{N}_i$. Set the penalty $\rho$, the number of local training steps $\tau$, {the number of iterations $K$}, and the local step-size $\gamma$, $\beta$.
 
	\For{$k = 0,1,\ldots, K-1$ every agent $i$}
 \CommentState{local training}

    \State $\phi_{i,k}^0 = x_{i,k}$, \saga{$r_{i, h, k}^{0}  = x_{i, k}$, for all $h \in \{ 1, \ldots. m_i \}$}
    \For{$t = 0, 1, \ldots, \tau-1$}

    \State draw the batch $\mathcal{B}_i$ uniformly at random

    \State \sgd{update the gradient estimator according to~\eqref{eq:sgd-gradient}}

    \State \saga{update the gradient estimator according to~\eqref{eq:saga-gradient}}

    \State update $\phi_{i,k}$ according to~\eqref{eq:local-training}

    \State \saga{if $h \in \mathcal{B}_i$ update $r_{i,h,k}^{t+1} = \phi_{i,k}^{t+1}$, else $r_{i,h,k}^{t+1} = r_{i,h,k}^{t}$}
	
    \EndFor
    
    \State $x_{i,k+1} = \phi_{i,k}^\tau$
    
	\CommentState{communication}
    \State transmit $z_{ij,k} - 2 \rho x_{i,k+1}$ to each neighbor $j \in \mathcal{N}_i$, and receive the corresponding transmissions
    
	\CommentState{auxiliary update}
	\State update $z_{ij,k+1}$ according to~\eqref{eq:admm-z}
    
	\EndFor
\end{algorithmic}
\end{algorithm}

\section{Convergence Analysis and Discussion}\label{sec:convergence}

In this section, we analyze the convergence rate of Algorithm~\ref{alg:lt-saga-admm} in both nonconvex and convex scenarios. Throughout, we will employ the following metric of convergence
\begin{equation}\label{eq:convergence-metric}
    \mathcal{D}_k = \mathbb{E} \left[ \left\|\nabla F\left(\bar{x}_{k}\right)\right\|^2+\frac{1}{\tau} \sum_{t = 0}^{\tau-1} \norm{\frac{1}{N} \sum_{i = 1}^N \nabla f_i\left(\phi_{i, k}^t\right)}^2 \right],
\end{equation}
where $F(x) = \frac{1}{N} \sum_{i = 1}^N f_i(x)$ and  {$\bar{x}_k = \frac{1}{N} \sum_{i = 1}^N x_{i,k}$.} If the agents converge to a stationary point of~\eqref{eq:optimization-problem}, then $\mathcal{D}_k \to 0$.
{
We note that this performance measure is standard in the literature on stochastic gradient methods and distributed optimization~\cite{alghunaim_local_2023,guo_randcom_2023}. Although it does not, in general, imply almost sure convergence of the sequence $\{ \mathcal{D}_k\}_{k=1}^K$, it provides meaningful performance guarantees in expectation. Specifically, if the index $K'$ is selected uniformly at random from ${1, \ldots, K}$, then $\mathbb{E}[ \mathcal{D}_{K'}] = \frac{1}{K} \sum_{k=0}^{K-1} \mathcal{D}_k$.
}

\subsection{Convergence with SGD}
We start by characterizing the convergence of Algorithm~\ref{alg:lt-saga-admm} when the agents use SGD during local training {(\algsgd)}. To this end, we make the following standard assumption on the variance of the gradient estimators, {see \textit{e.g.} \cite{liu_decentralized_2023,alghunaim_local_2023}}.

\begin{assumption} \label{as:variance}
For all $\phi \in \R^n$ the gradient estimators $g_i(\phi)$, $i \in \mathcal{V}$, in \eqref{eq:sgd-gradient} are unbiased and their variance is bounded by some $\sigma^2 > 0$:
$$
\begin{gathered}
\mathbb{E}\left[g_i(\phi) -\nabla f_i(\phi) \right]=0 \\
\mathbb{E}\left[\left\|g_i(\phi) - \nabla f_i(\phi)\right\|^2 \right] \leq \sigma^2
\end{gathered}.
$$
\end{assumption}

\smallskip

We are now ready to state our convergence results. All the proofs are deferred to the Appendix, {where Appendix~\ref{proof:sketch} provides a sketch of the proofs, followed by the full proofs.}
{We remark that prior analyses of distributed ADMM based on operator-theoretic approaches~\cite{bastianello_asynchronous_2021,bastianello_robust_2024} are not directly applicable to LT-ADMM, and the convergence proofs must therefore be specifically tailored to this algorithm.}

\begin{theorem}[Nonconvex case]\label{th:nonconvex_SGD}
Let Assumptions \ref{as:graph}, \ref{as:local-costs}, and \ref{as:variance} hold. If the {local step-sizes satisfy $\gamma \leq \mathcal{O}(\frac{ \lambda_l}{L\tau^2 })< \gamma_{\text{sgd}}\coloneqq \min _{i=1, 2,  \ldots, 6} \bar\gamma_i $ (see \eqref{gamma_sgd} in Appendix~\ref{sec:preliminary_definition} for the precise bound), and $ 1 / (\tau\lambda_u\rho) \leq \beta <  2 / (\tau\lambda_u\rho)$}, then the output of \algsgd satisfies:
\begin{equation}\label{theo:main}
\frac{1}{K} \sum_{k=0}^{K-1} \mathcal{D}_k \leq \mathcal{O}\left(\frac{ F(\bar{x}_{0})-F(x^*)}{K \gamma \tau}\right) + \mathcal{O}\left( \gamma \tau \sigma^2 \right)  + {\mathcal{O}\left( \frac{\|\widehat{\mathbf{d}}_{0}\|^2}{\rho^2 K N} \right) },
\end{equation}
where $x^*$ is a stationary point of~\eqref{eq:optimization-problem},   $\lambda_u$ is the largest eigenvalue of the graph $\mathcal{G}$'s Laplacian matrix,  $\lambda_l$ is the smallest nonzero eigenvalue of graph $\mathcal{G}$' Laplacian matrix, and $\| \widehat{\mathbf{d}}_0 \|$ is related to the initial conditions (see \eqref{eq:delta-hat}).
\end{theorem}

\smallskip

Theorem~\ref{th:nonconvex_SGD} shows that \algsgd converges to a neighborhood of a stationary point $x^*$ as $K \to \infty$. The radius of this neighborhood is proportional to the step-size $\gamma$, to the number of local training epochs $\tau$, and to the stochastic gradient variance $\sigma^2$.
The result can then be particularized to the convex case as follows.

\smallskip

\begin{corollary}[Convex case]\label{co:convex_SGD}
In the setting of Theorem~\ref{th:nonconvex_SGD}, with the additional Assumption~\ref{as:convex}, then the  output of \algsgd converges to a neighborhood of an optimal solution $x^*$ characterized by~\eqref{theo:main}.
\end{corollary}

\smallskip

\begin{remark}[Exact convergence]
Clearly, if we employ full gradients (and thus $\sigma = 0$), then these results prove exact convergence to a stationary/optimal point. This verifies that our algorithm design achieves convergence despite the use of approximate local updates.
\end{remark}

\subsection{Convergence with variance reduction}
The results of the previous section shows that only inexact convergence can be achieved when employing SGD. The following results show how Algorithm~\ref{alg:lt-saga-admm} achieves exact convergence when using variance reduction {(\algsaga)}.

\begin{theorem}[Nonconvex case]\label{th:nonconvex_SGD_VR}
Let Assumptions \ref{as:graph}, \ref{as:local-costs} hold. If the {local step-sizes satisfy $\gamma \leq \mathcal{O}(\frac{\lambda_l}{L \tau^3 }) < \gamma_{\text{vr}}\coloneqq \min _{i=1, 7, 8, \ldots, 15} \bar\gamma_i $ (see \eqref{gamma_saga} in Appendix~\ref{sec:preliminary_definition} for the precise bound), and $ 1 / (\tau\lambda_u\rho) \leq \beta < 2 / (\tau\lambda_u\rho)$}, then the output of \algsaga converges to a stationary point $x^*$ of~\eqref{eq:optimization-problem}, and in particular it holds:
\begin{equation}\label{theo:main_vr}
\frac{1}{K} \sum_{k=0}^{K-1} \mathcal{D}_k \leq  \mathcal{O}\left(\frac{ F\left(\bar{x}_{0}\right)-F(x^*)}{K \gamma \tau}\right) + {\mathcal{O}\left(\frac{\|\widehat{\mathbf{d}}_0 \|^2}{\rho^2 K} \right),}
\end{equation}
 where  $x^*$ is a stationary point of~\eqref{eq:optimization-problem},    $\lambda_u$ is the largest eigenvalue of the graph $\mathcal{G}$'s Laplacian matrix,  $\lambda_l$ is the smallest nonzero eigenvalue of graph $\mathcal{G}$' Laplacian matrix, and
$\| \widehat{\mathbf{d}}_0 \|$ is related to the initial conditions (see \eqref{eq:delta-hat}).
\end{theorem}

\smallskip

\begin{corollary}[Convex case]\label{co:convex_SGD_VR}
In the setting of Theorem~\ref{th:nonconvex_SGD_VR}, with the additional Assumption~\ref{as:convex}, 
then the  output of \algsaga converges to an optimal solution $x^*$, with rate characterized by~\eqref{theo:main_vr}.
\end{corollary}

\subsection{Discussion} \label{subsec:algorithm-discussion}

\subsubsection{Choice of step-size}
The upper bounds to the step-sizes of \algsgd and \algsaga (\eqref{gamma_sgd} and \eqref{gamma_saga} in Appendix~\ref{sec:preliminary_definition}), highlight a dependence on several \textit{features of the problem}.
In particular, the step-size bounds decrease as the smoothness constant $L$ increases, as is usually the case for gradient-based algorithms. Moreover, the bounds are proportional to the network connectivity, represented by the smallest nonzero eigenvalue of $\mathcal{G}$'s Laplacian (the algebraic connectivity $\lambda_l$). Thus, less connected graphs (smaller $\lambda_l$) result in smaller bounds.
Finally, we remark that the step-size bound for \algsaga is proportional to $\frac{m_l}{m_u} = \frac{\min_{i \in \mathcal{V}} m_i}{\max_{i \in \mathcal{V}} m_i}$, where $m_i$ is the number of data points available to agent $i$ (see~\eqref{eq:erm-cost}). This ratio can be viewed as a measure of heterogeneity between the agents. Smaller values of $\frac{m_l}{m_u}$ highlight higher imbalance in the amount of data available to the agents. The step-size bound thus is smaller for less balanced scenarios.

The step-size bounds also depend on the \textit{tunable parameters} $\tau$, the number of local updates, and $\rho$, the penalty parameter. Therefore, these two parameters can be tuned in order to increase the step-size bounds, which translates in faster convergence.

\subsubsection{Convergence rates}\label{subsec:convergence-rate}
As discussed in section~\ref{subsec:sota}, various distributed algorithms with variance reduction have been recently proposed, for example, \cite{li_variance_2022,xin_variance-reduced_2020} for strongly convex problems, and \cite{xin_fast_2022,xin_fast_2022b,jiang_distributed_2023} for nonconvex problems.
{Focusing on \cite{xin_fast_2022,xin_fast_2022b,jiang_distributed_2023}, we notice that their convergence rate is $\mathcal{O}(\frac{1}{K})$, while Theorem~\ref{th:nonconvex_SGD_VR} shows that \algsaga has rate of $\mathcal{O}(\frac{1}{\tau K})$. This shows that \textit{employing local training accelerates convergence}.}

Similarly to \algsaga, \cite{xin_fast_2022,li_variance_2022} also use batch gradient computations, \textit{i.e.}, they update a subset of components to estimate the gradient (see \eqref{eq:saga-gradient}). Interestingly, the step-size upper bound and, hence, the convergence rate in \cite{xin_fast_2022,li_variance_2022} depend on the batch size. On the other hand, our theoretical results are not affected by the batch size, since we use a different variance reduction technique. 

{We also remark that, as shown in \eqref{eq:converge} and \eqref{eq:convergence_vr} in the Appendix, better network connectivity (corresponding to larger $\lambda_l$) leads to smaller upper bounds on the right-hand side of the convergence results. This indicates that stronger network connectivity accelerates the convergence rate.
}

{Finally, the bound in Theorem~\ref{th:nonconvex_SGD} also highlights a trade-off: a larger $\gamma$ accelerates convergence through the term $\mathcal{O}\!(\tfrac{1}{K\gamma\tau})$, but it also enlarges the steady-state neighborhood via the term $\mathcal{O}\!\left(\gamma\tau\sigma^2\right)$. Thus, $\gamma$ must be tuned to balance convergence speed and steady-state precision -- and a similar discussion holds for $\tau$.
This trade-off is also explored in the numerical results of section~\ref{subsec:numerical-tuning-parameters}.}

\subsubsection{Choice of variance reduction mechanism}\label{subsubsec:sarah-saga}
In variance reduction, we distinguish two main classes of algorithms: those that need to periodically (or randomly) perform a full gradient evaluation (SARAH-type \cite{nguyen2017sarah}), and those that do not (SAGA-type \cite{defazio2014saga}). In distributed learning, SARAH-type algorithms were proposed in \textit{e.g.}, \cite{xin_fast_2022,xin_variance-reduced_2020}, while SAGA-type algorithms in \textit{e.g.} \cite{xin_variance-reduced_2020}.
Also the proposed \algsaga requires a periodic full gradient evaluation, as the agents re-initialize their local gradient memory at the start of each local training (since they set $r_{i, h, k}^{0}  = x_{i, k}$).
Clearly, periodically computing a full gradient significantly increases the computational complexity of the algorithm. Thus, one can design a SAGA-type variant of \algsaga by removing the gradient memory re-initialization at the start of local training (choosing now $r_{i, h, k}^{0}  = r_{i, h, k-1}^{\tau}$).
This variant is computationally cheaper and shows promising empirical performance, see the results for \algsagatwo in section~\ref{sec:numerical}. However, using the outdated gradient memory leads to a more complex theoretical analysis, which we leave for future work.

\subsubsection{Uncoordinated parameters}\label{subsec:uncoordinated}
{
In principle, the agents could employ \textit{uncoordinated} parameters, depending on their available resources (\textit{e.g.}, heterogeneous computational capabilities). For instance, different agents could adopt distinct local solvers, numbers of updates (see results in Section~\ref{subsec:numerical-tuning-parameters}), and batch sizes. Alternatively, they could use the same solver but with step-sizes tailored to the smoothness of their local cost functions.
}

\section{Numerical Results}\label{sec:numerical}
In this section we compare the proposed algorithms with the state of the art, applying them to a classification problem with nonconvex regularization, characterized by \cite{alghunaim_local_2023}:
$$
    f_i(x) = \frac{1}{m_i} \sum_{h = 1}^{m_i} \log\left( 1 + \exp\left( - b_{i,h} a_{i,h}^\top x \right) \right) + \epsilon \sum_{\ell = 1}^n \frac{[x]_\ell^2}{1 + [x]_\ell^2}
$$
where $[x]_\ell$ is the $\ell$-th component of $x \in \R^n$, and $a_{i,h} \in \R^n$ and $b_{i,h} \in \{ -1, 1 \}$ are the pairs of feature vector and label. {As data we use $8 \times 8$ gray-scale images of handwritten digits\footnote{{From \url{https://doi.org/10.24432/C50P49}.}}, with pixels normalized to $[0, 1]$; we divide the images in the two classes `even' and `odd'. We choose a ring graph with $N = 10$, and have $n = 64$, $m_i = 180$, $\epsilon = 0.01$; the initial conditions are randomly chosen as $x_{i,0} \sim \mathcal{N}(0, 100 I_n)$.}
We use stochastic gradients with a batch of $|\mathcal{B}| = 1$. All results are averaged over $10$ Monte Carlo iterations. For the algorithms with local training we select $\tau = 2$. We also tune the step-sizes of all algorithms to ensure best performance.
Finally, as performance metric we employ $\norm{\nabla F(\bar{x}_k)}^2$, which is zero if the agents have reached a solution of~\eqref{eq:optimization-problem}\footnote{{
We choose this metric as it can be defined for all algorithms 
considered in the comparison, whereas $\mathcal{D}_k$ is defined specifically for Algorithm~\ref{alg:lt-saga-admm}.
}}.
{The simulations are implemented in Python and run on a Windows laptop with Intel i7-1265U and 16GB of RAM.}

\subsection{Comparison with the state of the art}
{
We start by comparing \algsgd and \algsaga with local training algorithms LED \cite{alghunaim_local_2023} and K-GT \cite{liu_decentralized_2023}, as well as variance reduction algorithms GT-SARAH \cite{xin_fast_2022}, and GT-SAGA \cite{xin_variance-reduced_2020}. We also compare with the alternative version \algsagatwo discussed in section~\ref{subsubsec:sarah-saga}.
When evaluating the performance, we account for the computation time of each algorithm, rather than the more commonly used iteration count.
In particular, letting $\tgrad$ be the time for a component gradient evaluation ($\nabla f_{i,h}$), and $\tcomm$ the time for a round of communications, Table~\ref{tab:time-comparison} reports the computation time incurred by each algorithm over the course of $\tau$ iterations.

\begin{table}[!ht]
\centering
\caption{Computation time of the algorithms over $\tau$ iterations.}
\label{tab:time-comparison}
\begin{tabular}{cc}
    \hline
    Algorithm [Ref.] & Time \\
    \hline
    LED \cite{alghunaim_local_2023} \& K-GT \cite{liu_decentralized_2023} & $\tau \tgrad + 2 \tcomm$ \\
    GT-SARAH \cite{xin_fast_2022} & $(m_i + \tau - 1) \tgrad + 2 \tau \tcomm$ \\
    GT-SAGA \cite{xin_variance-reduced_2020} & $\tau \left( \tgrad + 2 \tcomm \right)$ \\
    \hline
    \algsgd \& \algsagatwo & $\tau \tgrad + \tcomm$ \\
    \algsaga & $(m_i + \tau - 1) \tgrad + \tcomm$ \\
    \hline
\end{tabular}
\end{table}

We start by comparing in Table~\ref{tab:comparison-vr} the algorithms with variance reduction, in terms of the computation time they require to reach $\norm{\nabla F(\bar{x}_k)}^2 < 10^{-7}$, that is, to reach a stationary point up to numerical precision.
\begin{table}[!ht]
\centering
\caption{Comparison of computation time for variance-reduced algorithms to reach $\norm{\nabla F(\bar{x}_k)}^2 < 10^{-7}$.}
\label{tab:comparison-vr}
\begin{tabular}{cccc}
    \hline
    Algorithm [Ref.] & $\tgrad / \tcomm = 0.1$ & $\tgrad / \tcomm = 1$ & $\tgrad / \tcomm = 10$ \\
    \hline
    GT-SARAH \cite{xin_fast_2022} & $7.57 \times 10^5$ & $6.33 \times 10^6$ & $6.20 \times 10^7$ \\
    GT-SAGA \cite{xin_variance-reduced_2020} & $1.55 \times 10^5$ & $2.21 \times 10^5$ & $8.85 \times 10^5$ \\
    \hline
    \algsaga & $6.04 \times 10^5$ & $5.76 \times 10^6$ & $5.73 \times 10^7$ \\
    \algsagatwo & $3.81 \times 10^4$ & $9.52 \times 10^4$ & $6.66 \times 10^5$ \\
    \hline
\end{tabular}
\end{table}
We see that, depending on the ratio $\tgrad / \tcomm$, their relative speed of convergence changes. When gradient computations are cheaper than communications ($\tgrad / \tcomm = 0.1$), the proposed algorithm \algsaga (and \algsagatwo) outperform both GT-SARAH and GT-SAGA, since the latter do not employ local training. \textit{This testifies to the benefit of employing local training in scenarios where communications are expensive}.
As the ratio $\tgrad / \tcomm$ increases to $1$ and then $10$, we see how \algsaga and GT-SARAH, on the one hand, and \algsagatwo and GT-SAGA, on the other hand, tend to align in terms of performance, as the bulk of the computation time is now due to gradient evaluations, of which the two pairs of algorithms have a similar number (see Table~\ref{tab:time-comparison}). Nonetheless, local training still gives an edge to the proposed algorithms.

The remaining algorithms (\algsgd, LED, K-GT) do not guarantee exact convergence as they do not employ variance reduction. Thus in Table~\ref{tab:comparison-sgd} we report the asymptotic value of $\norm{\nabla F(\bar{x}_k)}^2$ achieved by the different methods.
\begin{table}[!ht]
\centering
\caption{Comparison of algorithms without variance reduction.}
\label{tab:comparison-sgd}
\begin{tabular}{cc}
    \hline
    Algorithm [Ref.] & $\norm{\nabla F(\bar{x}_K)}^2$ \\
    \hline
    LED \cite{alghunaim_local_2023} & $1.29 \times 10^{-3}$ \\
    K-GT \cite{liu_decentralized_2023} & $2.01 \times 10^{-3}$ \\
    \hline
    \algsgd & $1.07 \times 10^{-3}$ \\
    \hline
\end{tabular}
\end{table}
The algorithms have close performance, with the proposed \algsgd outperforming the state of the art slightly, that is, converging closer to a stationary point.
}

\subsection{Tuning the parameters}\label{subsec:numerical-tuning-parameters}
{
In this section we focus on evaluating the impact of the proposed algorithms' tunable parameters.
As discussed in section~\ref{subsec:convergence-rate}, the step-size of \algsgd regulates both the speed of convergence and how close it converges to a stationary point. In Table~\ref{tab:role-step-size} then we apply different step-sizes and evaluate both the asymptotic value of $\norm{\nabla F(\bar{x}_k)}^2$ and the computation time needed for \algsgd to reach such value.
\begin{table}[!ht]
\centering
\caption{Performance of \algsgd for different $\gamma$.}
\label{tab:role-step-size}
\begin{tabular}{ccc}
    \hline
    $\gamma$ & $\norm{\nabla F(\bar{x}_K)}^2$ & Computation time \\
    \hline
    $0.1$ & $6.01 \times 10^{-5}$ & $4.80 \times 10^4$ \\
    $1$ & $5.79 \times 10^{-4}$ & $3.22 \times 10^4$ \\
    $2$ & $1.07 \times 10^{-3}$ & $2.27 \times 10^2$ \\
    $3$ & $1.72 \times 10^{-3}$ & $2.38 \times 10^2$ \\
    $4$ & $2.68 \times 10^{-3}$ & $7.9 \times 10^1$ \\
    $5$ & $4.43 \times 10^{-3}$ & $4.10 \times 10^1$ \\
    \hline
\end{tabular}
\end{table}
As expected, a smaller step-size leads to a smaller asymptotic distance from the stationary point, while a larger step-size improves the speed of convergence. 

We turn now to \algsaga and evaluate its speed of convergence for different numbers of local training epochs $\tau$. Figure~\ref{fig:admm-time} reports the computation time to reach $\norm{\nabla F(\bar{x}_k)}^2 < 10^{-7}$.
\begin{figure}[!ht]
    \centering
    \includegraphics[scale=0.5,trim={0 0.365cm 0 0},clip]{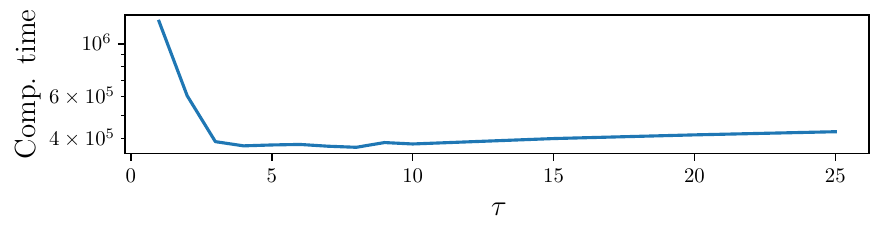}
    \caption{Computation time for \algsaga to reach $\norm{\nabla F(\bar{x}_k)}^2 < 10^{-7}$ for different numbers of local training epochs $\tau$.}
    \label{fig:admm-time}
\end{figure}
Interestingly, it appears that there is a finite optimal value ($\tau = 8$), while smaller and larger values lead to slower convergence.

Finally, as discussed in section~\ref{subsec:uncoordinated}, we can actually choose uncoordinated parameters in \algsaga. In Table~\ref{tab:uncoordinated} then we test the use of different $\tau_i$, $i \in \{ 1, \ldots, N \}$, for different agents.
\begin{table}[!ht]
\centering
\caption{Computation time of \algsaga with uncoordinated numbers of local training epochs.}
\label{tab:uncoordinated}
\begin{tabular}{cc}
    \hline
    $\tau_i$ & Computation time \\
    \hline
    $2$ $\forall i$ & $6.04 \times 10^{5}$ \\
    $\begin{cases}
        2 & i < N/2 \\
        5 & i\geq N/2
    \end{cases}$ & $5.80 \times 10^{5}$ \\
    $5$ $\forall i$ & $3.75 \times 10^{5}$ \\
    \hline
\end{tabular}
\end{table}
In particular, we compare the computation time required to reach $\norm{\nabla F(\bar{x}_k)}^2 < 10^{-7}$ in two coordinated scenarios and an uncoordinated scenario where half of the agents are ``slow'' ($\tau_i = 2$) and half are ``fast'' ($\tau_i = 5$). Interestingly, the algorithm still converges even with uncoordinated parameters, with the presence of ``fast'' agents improving the performance.
}

\section{Concluding Remarks} \label{sec:conclusions}

In this paper, we considered (non)convex distributed learning problems. In particular, to address the challenge of expensive communication, we proposed two communication-efficient algorithms, \algsgd and \algsaga, that use local training. The algorithms employ SGD and SGD with variance reduction, respectively. We have shown that \algsgd converges to a neighborhood of a stationary point, while \algsaga converges exactly. We have thoroughly compared our algorithms with the state of the art, both theoretically and in simulations.
{Future research will focus on analyzing convergence for strongly convex problems and on extending our algorithmic framework to asynchronous scenarios and to the broader class of composite problems, as in~\cite{hong2016convergence}.}

\appendices

\section{Proof sketch of the main theorems}\label{proof:sketch}

\subsection{Proof sketch of Theorem~\ref{th:nonconvex_SGD}}
{

\paragraph*{Step 1 (Lemma~\ref{lem:devitaion_aver})}
 Reformulate the algorithm into a compact linear dynamical system, in which $\mathbf{h}_k$ contains all nonlinearities.
 Decompose the system into average and deviation components via a projection matrix $\widehat{\mathbf{Q}}$. 
 Use graph connectivity to show that the linear part of the deviation system 
 \begin{equation*}
 \widehat{\mathbf{d}}_{k+1}=\mathbf{\Delta} \widehat{\mathbf{d}}_k- \widehat{\mathbf{h}}_{k}
 \end{equation*} 
 is stable ($\|\mathbf{\Delta}\|  = 1 - {\lambda_l\rho \tau \beta} /{2} < 1$) when  $\beta <  \frac{2}{\tau\lambda_u\rho}$ is satisfied.

\paragraph*{Step 2 (Lemmas~\ref{lem:phi_k},~\ref{lemma:d_k})}
 Bound the error from local training steps: the deviation of local states $\boldsymbol{\Phi}_k^t$ from the global average $\mathbf{\bar{X}}_k$  
 as in \eqref{phi_sgd}.
\begin{align*}
\mathbb{E} \left[ \|\widehat{\mathbf{\Phi}}_k\|^2 \right] &\leq 
\left( \frac{72\beta \tau^2}{\lambda_l\rho}  + 144\tau^3\beta^2  \right) \mathbb{E} [  \|\widehat{\mathbf{d}}_{k}\|^2]  + 4N\tau^2\gamma^2 \sigma^2
\\& + 16\tau^3\gamma^2 \mathbb{E} [  \| \nabla F(\bar{x}_k) \|^2 ].
\end{align*}
 Incorporate this bound into the perturbation term $\widehat{\mathbf{h}}_k$ to derive a recursive inequality for the deviation system as in \eqref{d_k},
 \begin{align*}
&\mathbb{E} [ \|\widehat{\mathbf{d}}_{k+1}\|^2 ] 
\\&\leq ( \delta + \frac{c_0}{1-\delta}) \mathbb{E} [ \|\widehat{\mathbf{d}}_{k}\|^2 ] 
+ \frac{c_1}{1-\delta}  \mathbb{E}[ \| \sum_t \overline{\nabla F}(\boldsymbol{\Phi}^t_k) \|^2 ]
\\& + \frac{c_2}{1-\delta}\mathbb{E}[\| \nabla F(\bar{x}_k) \|^2] + \frac{c_3}{1-\delta} \sigma^2,
\end{align*} 
    where sufficiently small $\gamma$ ensuring stability.


\paragraph*{Step 3 (Theorem~\ref{th:nonconvex_SGD})}
Apply the smoothness inequality to the averaged iterate 
\begin{equation*}
\bar{x}_{k+1} - x^* = \bar{x}_{k}  - x^*   -\frac{\gamma}{N} \sum_{t=0}^{\tau -1} \sum_{i=1}^{N} g_i(\phi_{i,k}^{t}), 
\end{equation*}
show a descent in the global objective $ F\left(\bar{x}_{k}\right)$.
    \begin{align*}
 & \mathbb{E} [ F\left(\bar{x}_{k+1}\right) ] 
 \leq   \mathbb{E} [ F\left(\bar{x}_{k}\right) ]-\frac{\gamma \tau}{2} \mathbb{E} [ \left\|\nabla F\left(\bar{x}_{k}\right)\right\|^2 ]
 \\& -\frac{\gamma}{2}(1 - 2 \gamma \tau L) \sum_t \mathbb{E} [ \|\frac{1}{N} \sum_i \nabla f_i\left(\phi_{i, k}^t\right) \|^2 ] 
 \\& +\frac{\gamma L^2}{2 N} \mathbb{E} [ \|\widehat{\mathbf{\Phi}}_k\|^2 ] + { \gamma^2 \tau^2 L \sigma^2}.
\end{align*}
 Sum over iterations, pick an appropriate step-size, and use the bounds obtained in Lemmas~\ref{lem:phi_k},~\ref{lemma:d_k}, leading to the convergence result.

}

\subsection{Proof sketch of Theorem~\ref{th:nonconvex_SGD_VR}}
{
The proof of Theorem~\ref{th:nonconvex_SGD_VR} follows a structure similar to that of Theorem~\ref{th:nonconvex_SGD}. The key distinction lies in the treatment of the gradient variance. We bound the gradient variance with $t_k$, 
\begin{align*}
  \mathbb{E} [ \sum_i \sum_t \|   g_i\left(\phi_{i, k}^t\right) -   \nabla f_i(\phi_{i, k}^t) \|^2  ] \leq 2L^2 \|\widehat{\mathbf{\Phi}}_k \|^2 + 2 L^2 \mathbb{E}[{t}_k], 
  \end{align*}
  using $ \mathbb{E} [ \|a- \mathbb{E}[a] \|^2 ] \leq \mathbb{E} [ \| a \|^2 ]$ with $ a = \nabla f_{i, h}( \phi_{i, k}^t )-\nabla f_{i, h}(r_{i, h, k}^{t}) $.
  And show that $t_k$ can be bound by the deviation bound $\| \widehat{\mathbf{d}}_k\|$ and the global gradient at the average state $\nabla F(\bar{x}_k)$ as in Lemma~\ref{lemma:tk}. And we further bound the deviation bound $\| \widehat{\mathbf{d}}_k\|$ with $t_k$ as in Lemma~\ref{lem:vr_d_k}. As a result, the final convergence expression no longer depends on a constant $\sigma$, and we can have exact convergence. 
}

\section{Preliminary Analysis} \label{preliminary}
In this section we summarize the step-size bounds, and present preliminary results underpinning Theorems \ref{th:nonconvex_SGD} and \ref{th:nonconvex_SGD_VR}.

\subsection{Step-size bounds}\label{sec:preliminary_definition}
The step-size upper bounds for \algsgd and \algsaga are, respectively:
\begin{equation} \label{gamma_sgd}
	\bar\gamma_{\text{sgd}}\coloneqq \min _{i=1, 2,  \ldots, 6} \bar\gamma_i, 
 \end{equation}
 \begin{equation} \label{gamma_saga}
 \bar\gamma_{\text{vr}}\coloneqq \min _{i=1, 7, 8, \ldots, 15} \bar\gamma_i,
\end{equation}
where:
\begin{align*}
 \bar{\gamma}_1&\coloneqq  \min \left\lbrace  1, \frac{1}{L\tau 2\sqrt{2}}
  \right\rbrace, \ \bar{\gamma}_2 \coloneqq \frac{\sqrt{3}}{8L\tau}, 
	\\
 \bar{\gamma}_3&\coloneqq {\frac{3}{8 L \tau}}, 
 \quad 
 \bar{\gamma}_4 \coloneqq \frac{ \lambda_l}{\lambda_u L\sqrt{16( 1+ 2\rho^2 \|\Tilde{\mathbf{L}} \|^2)  \tau  \|\widehat{\mathbf{V}}^{-1} \|^2 \beta_0}  }, 
  \\  \bar{\gamma}_5&\coloneqq  \frac{ \sqrt{\lambda_l } }{4 \tau \sqrt{\lambda_u L} \sqrt[4]{c_4 ( 1+ 2\rho^2 \|\Tilde{\mathbf{L}} \|^2)  2  \|\widehat{\mathbf{V}}^{-1} \|^2 } },
  \\\bar{\gamma}_6 & \coloneqq  \frac{ \sqrt{ \lambda_l\beta } }{2 \sqrt{\tau \lambda_uL} \sqrt[4]{c_4 6  N\|\widehat{\mathbf{V}}^{-1} \|^2} }, 
  \\
  \bar{\gamma}_7 &\coloneqq  \frac{1}{4\tau L \sqrt{3}},
\bar{\gamma}_8 \coloneqq   \sqrt{\frac{m_l}{32L^2m_u}},  \bar{\gamma}_9 \coloneqq \sqrt{\frac{1}{12L^2}},
\\\bar{\gamma}_{10} &\coloneqq \sqrt{\frac{m_l}{512 m_u \tau^3 L^2}}, 
\bar{\gamma}_{11} \coloneqq \sqrt{ \frac{3\tau}{8  \kappa_3} }, \bar{\gamma}_{12} \coloneqq \frac{3}{8\tau L} ,
\\ \bar{\gamma}_{13} & \coloneqq \frac{\lambda_l}{\lambda_u \sqrt{8(\kappa_0  \tilde{\beta}_0 + 2(\tilde{s}_0 +\tilde{s}_1) (\kappa_1+ 32 \tau^2 L^2 \kappa_0))}},
\\  \bar{\gamma}_{14} & \coloneqq \sqrt[3]{ \frac{\lambda_l^2 \beta^2}{768L^2N\lambda_u^2 \kappa_4}},  \quad
  \bar{\gamma}_{15}  \coloneqq  \sqrt[3]{\frac{\lambda_l^2 \tau}{128\lambda_u^2 \kappa_4\kappa_2 } } 
\end{align*}
These bounds depend on the following quantities. $d_u = \max \{ |\mathcal{N}_i| \}_{{i \in \mathcal{V}} }$ denotes the maximum agents' degree.
$\widehat{\mathbf{V}}$ is defined in \eqref{eq: v_hat}. 
 $\lambda_u$ is the largest eigenvalue of the graph $\mathcal{G}$'s Laplacian matrix,  $\lambda_l$ is the smallest nonzero eigenvalue of graph $\mathcal{G}$'s Laplacian matrix. 
We denote $m_u=\max_{i=1,...,N} m_i$ and $m_l=\min_{i=1,...,N} m_i$, where $m_i$ is the number of local data points of agent $i$.
Additionally, we have the following definitions, used both in the upper bound above and throughout the convergence analysis:
\begin{align*}
\beta_0 & \coloneqq  \frac{72\beta \tau^2}{\lambda_l\rho}  + 144\tau^3\beta^2,
\\ \tilde{\beta}_0 & \coloneqq \frac{72\beta \tau^2}{\lambda_l\rho}  + 144\tau^3\beta^2,
\\c_4 &\coloneqq  \frac{ 4 L^2}{ N}  \left( \frac{72\beta \tau}{\lambda_l\rho}  + 144\tau^2\beta^2  \right) 
\\\kappa_0&  \coloneqq   ( 1+ 2\rho^2 \|\Tilde{\mathbf{L}} \|^2) 6\tau L^2 \|\widehat{\mathbf{V}}^{-1} \|^2 +  \frac{6 L^2}{\beta^2} 2\tau L^2\|\widehat{\mathbf{V}}^{-1} \|^2, 
\\\kappa_1 &  \coloneqq  ( 1+ 2\rho^2 \|\Tilde{\mathbf{L}} \|^2) 4\tau L^2 \|\widehat{\mathbf{V}}^{-1} \|^2 +\frac{ 6 L^2 }{\beta^2}  2 \tau L^2\|\widehat{\mathbf{V}}^{-1} \|^2, 
\\
 \kappa_2& \coloneqq  16\tau^3N \kappa_0 + 2\tilde{s}_2(\kappa_1+ 32 \tau^2  L^2 \kappa_0),
 \\  \tilde{s}_0 &\coloneqq \frac{36\beta \tau^2 m_u }{\lambda_l\rho} +\frac{144 \tau^2 m_u}{m_l}\beta^2,
\\ \tilde{s}_1 & \coloneqq \left( \frac{72\beta \tau^2}{\lambda_l\rho}  +144\tau^3\beta^2 \right) \frac{8m_u \tau}{m_l},
\\ \tilde{s}_2 & \coloneqq \frac{16N  m_u \tau^2}{m_l} + \frac{8m_u \tau}{m_l}16 \tau^3 N,
\\ \kappa_3 &\coloneqq 16\tau^3N (\frac{L^2}{2 N}+ 2 \tau L^3) \\&+ 2\tilde{s}_2(2 \tau L^3+ 32 \tau^2  L^2 (\frac{L^2}{2 N}+ 2\tau L^3)),
\\\kappa_4 & \coloneqq (\frac{ L^2}{2 N}+ 2 \tau L^3)  \tilde{\beta}_0 \\&+ 2(\tilde{s}_0 +\tilde{s}_1) (2\tau L^3+ 32 \tau^2  L^2 (\frac{ L^2}{2 N}+ 2 \tau L^3)),
\end{align*}
where $ \frac{1}{\tau\lambda_u\rho} \leq \beta <  \frac{2}{\tau\lambda_u\rho}$.
\subsection{Preliminary transformation}
We start by rewriting the algorithm in a compact form. To this end, we introduce the following auxiliary variables: $\mathbf{Z} = \operatorname{col}\{z_{ij}\}_{i,j \in \mathcal{E}} 
    $, $ \mathbf{\Phi}_k^t = \operatorname{col}\{\phi^t_{1,k}, \phi_{2,k}^t, ..., \phi_{N,k}^t\}
   $, $ G(\Phi_k^t) = \operatorname{col}\{ g_1(\phi_{1,k}^t), g_2(\phi_{2,k}^t),..., g_N(\phi_{N,k}^t)\}
    $, $ \mathrm{F}(\mathbf{X}) = \operatorname{col}\{ f_1(x_1),  f_2(x_2),...,  f_N(x_N)\}
    $, $ F({x}) = \frac{1}{N}\sum_{i=1}^N f_i({x})$.
Define 
$
\mathbf{A}= \operatorname{blk\,diag}\{ \mathbf{1}_{d_i} \}_{i \in \mathcal{V}} \otimes \mathbf{I}_n \in \mathbb{R}^{Mn \times Nn},
$
where $d_i = |\mathcal{N}_i|$ is the degree of node $i$, and $M = \sum_i |\mathcal{N}_i|$.
$\mathbf{P} \in \mathbb{R}^{Mn \times Mn}$ is a permutation matrix that swaps $e_{ij}$  with $e_{ji}$. If there is an edge between nodes $i$, $j$, then $A^T[i,:]PA[:,j] = 1$, otherwise $A^T[i,:]PA[:,j] = 0$.
 Therefore $ \mathbf{A}^T\mathbf{P}\mathbf{A} = \Tilde{\mathbf{A}}$ is the adjacency matrix. 

The compact form of {\algsgd} and {\algsaga} then is:
\begin{subequations}\label{eq:compact-admm}
\begin{equation}
        \mathbf{X}_{k+1} = \mathbf{X}_{k} -\sum_{t=0}^{\tau -1}( \gamma   G(\mathbf{\Phi}_k^t) + {\beta}(\rho \mathbf{A}^T\mathbf{A}\mathbf{X}_k - \mathbf{A}^T \mathbf{Z}_k ) ) \label{eq:compact-admm-x}
\end{equation}
\begin{equation}
    \mathbf{Z}_{k+1} =  \frac{1}{2}\mathbf{Z}_{k} - \frac{1}{2} \mathbf{P}\mathbf{Z}_k+ \rho \mathbf{P}\mathbf{A}\mathbf{X}_{k+1}. \label{eq:compact-admm-z}
\end{equation}
\end{subequations}
Moreover, we introduce the following useful variables
\begin{equation} \label{y_tilde_y}
\begin{aligned}
&\mathbf{Y}_k= \mathbf{A}^T \mathbf{Z}_{k} - \frac{\gamma}{\beta} \nabla F(\mathbf{\bar{X}}_k) -\rho \mathbf{D} \mathbf{X}_k
\\&
\Tilde{\mathbf{Y}}_k= \mathbf{A}^T \mathbf{P}\mathbf{Z}_{k} +  \frac{\gamma}{\beta} \nabla \mathrm{F}(\bar{\mathbf{X}}_k) - \rho \mathbf{D} \mathbf{X}_k,
\end{aligned}
\end{equation}
where $\bar{\mathbf{X}}_k = \mathbf{1}_N \otimes \bar{x}_k$, with $\bar{x}_k = \frac{1}{N} \mathbf{1}^T \mathbf{X}_k$, and $\mathbf{D} = \mathbf{A}^T\mathbf{A} = \operatorname{diag}\{ d_i \mathbf{I}_n \}_{i \in \mathcal{V}}$ is the degree matrix.

Multiplying both sides of \eqref{eq:compact-admm-z} by $\mathbf{1}^T$, and using the initial condition, we obtain $\mathbf{1}^T\mathbf{A}^T\mathbf{Z}_{k+1} = \rho \mathbf{1}^T \mathbf{D} \mathbf{X}_{k+1}$ for all $k \in \N$.
As a consequence $\bar{\mathbf{Y}}_k =  \frac{\gamma}{\beta} \mathbf{1} \otimes \frac{1}{N}  \mathbf{1}^T\nabla \mathrm{F}(\bar{\mathbf{X}}_k) =  \frac{\gamma}{\beta} \mathbf{1} \otimes \frac{1}{N} \sum_{i}\nabla f_i(\bar{x}_k)$, and \eqref{eq:compact-admm} can be further rewritten as
\begin{equation}\label{eq:compact-admm-2}
\begin{aligned}
   & \begin{bmatrix}
     \mathbf{X}_{k+1}\\
    \mathbf{Y}_{k+1}\\
   \Tilde{\mathbf{Y}}_{k+1}
\end{bmatrix} = \begin{bmatrix}
    \mathbf{I}  &  \beta \tau \mathbf{I} & \mathbf{0}  \\
     \rho \Tilde{\mathbf{L}}  &  \rho \Tilde{\mathbf{L}}\beta \tau +  \frac{1}{2} \mathbf{I}  & - \frac{1}{2}\mathbf{I}  \\
   \mathbf{0}  &   - \frac{1}{2}\mathbf{I} &  \frac{1}{2} \mathbf{I}
\end{bmatrix} \otimes \mathbf{I}_n \begin{bmatrix}
     \mathbf{X}_{k}\\
    \mathbf{Y}_{k}\\
   \Tilde{\mathbf{Y}}_{k}
\end{bmatrix} 
-   \mathbf{h}_k, 
\end{aligned}
\end{equation}
where 
\begin{equation} \label{eq: matrix_L}
\Tilde{\mathbf{L}} = \tilde{\mathbf{A}}- \mathbf{D}
\end{equation}
and $$
\begin{aligned}
\mathbf{h}_k &= 
\large[ \gamma  \sum_{t=0}^{\tau -1}( \nabla G(\mathbf{\Phi}_k^t) -  \nabla \mathrm{F}(\bar{\mathbf{X}}_k) ) ; \\&
\gamma  \rho  \Tilde{\mathbf{L}}\sum_{t=0}^{\tau -1}( \nabla G(\mathbf{\Phi}_k^t) -  \nabla \mathrm{F}(\bar{\mathbf{X}}_k)   ) + \frac{\gamma}{\beta} ( \nabla F(\bar{\mathbf{X}}_{k+1}) -  \nabla F(\bar{\mathbf{X}}_{k}) ) ;\\& \frac{\gamma}{\beta} ( -\nabla F(\bar{\mathbf{X}}_{k+1}) +  \nabla F(\bar{\mathbf{X}}_{k}) ) \large].
\end{aligned}$$
We remark that~\eqref{eq:compact-admm-2} can be interpreted as a linear dynamical system, with the non-linearity of the gradients as input in $\mathbf{h}_k$.

\subsection{Deviation from the average} \label{sec:devitaion_aver}
The following lemma illustrates how far the states deviate from the average and will be used later in the proofs of Lemmas~\ref{lem:phi_k} and~\ref{lem:vr_d_k}.
\begin{lemma} \label{lem:devitaion_aver}
 Let Assumption~\ref{as:graph} hold, when  $\beta <  \frac{2}{\tau\lambda_u\rho}$,   
 \begin{equation} \label{X_Y_d} 
\|  \bar{\mathbf{X}}_k- \mathbf{X}_k \|^2 \leq \frac{18\beta \tau}{\lambda_l\rho} \|  \widehat{\mathbf{d}}_k\|^2, \quad \|  \bar{\mathbf{Y}}_k- \mathbf{Y}_k \|^2 \leq 9 \|  \widehat{\mathbf{d}}_k \|^2,
\end{equation}
and 
\begin{equation} \label{d_k_0}
\|\widehat{\mathbf{d}}_{k+1}\|^2  \leq \delta \|\widehat{\mathbf{d}}_{k}\|^2 +
\frac{1}{1-\delta} \|\mathbf{\widehat{h}}_{k}\|^2
\end{equation}
where $\delta = 1 - {\lambda_l\rho \tau \beta}/{2} <1$,
{$
\widehat{\mathbf{d}}_k = \widehat{\mathbf{V}}^{-1}
\begin{bmatrix}
\widehat{\mathbf{Q}}^T  \mathbf{X}_{k};
   \widehat{\mathbf{Q}}^T \mathbf{Y}_{k};
   \widehat{\mathbf{Q}}^T \Tilde{\mathbf{Y}}_{k}
\end{bmatrix}
$, $\widehat{\mathbf{Q}}$ and $\widehat{\mathbf{V}}^{-1}$ are matrices used to define the deviation term $\widehat{\mathbf{d}}_k$. }
\end{lemma}
\begin{proof}
By Assumption~\ref{as:graph}, graph $\mathcal{G}$ is undirected and connected, hence its Laplacian $ -\Tilde{\mathbf{L}}$ is  symmetric; moreover, it has one zero eigenvalue with eigenvector $\boldsymbol{1}$, with all eigenvalues being positive.
Denote by $\widehat{\mathbf{Q}} \in \mathbf{R}^{N \times (N-1)}$ the matrix satisfying $\widehat{\mathbf{Q}} \widehat{\mathbf{Q}} ^T=\mathbf{I}_N-\frac{1}{N} \mathbf{1 1}{ }^T$,  $\widehat{\mathbf{Q}} ^T \widehat{\mathbf{Q}} =\mathbf{I}_{N-1}$  and $\mathbf{1}^T \widehat{\mathbf{Q}} =0$, $\widehat{\mathbf{Q}} ^T \mathbf{1}=0$.
We have that
\begin{equation}\label{Q}
\widehat{\mathbf{Q}} ^T \Tilde{\mathbf{L}} = \widehat{\mathbf{Q}} ^T \Tilde{\mathbf{L}}(\mathbf{I}_N-\frac{1}{N} \mathbf{1 1}{ }^T) = \widehat{\mathbf{Q}}^T \Tilde{\mathbf{L}}\widehat{\mathbf{Q}} \widehat{\mathbf{Q}} ^T.
\end{equation}
Additionally, it holds that
$\|\widehat{\mathbf{Q}} ^T \mathbf{X}_k\|^2= \mathbf{X}_k^T \widehat{\mathbf{Q}}   \widehat{\mathbf{Q}}^T \widehat{\mathbf{Q}} \widehat{\mathbf{Q}}^T \mathbf{X}_k= \left\| \widehat{\mathbf{Q}}  \widehat{\mathbf{Q}}^T \mathbf{X}_k\right\|^2=\left\|\mathbf{X}_k-\bar{{\mathbf{X}}}_k\right\|^2$, and $\|\widehat{\mathbf{Q}}\|=1$.
Multiplying both sides of \eqref{eq:compact-admm-2} by $\widehat{\mathbf{Q}}^T $ and using \eqref{Q} yields:
\begin{equation} \label{Q_2}
\begin{aligned}
    \begin{bmatrix}
   \widehat{\mathbf{Q}}^T  \mathbf{X}_{k+1}\\
   \widehat{\mathbf{Q}}^T \mathbf{Y}_{k+1}\\
   \widehat{\mathbf{Q}}^T \Tilde{\mathbf{Y}}_{k+1}
\end{bmatrix}& =  (\mathbf{\Theta} \otimes \mathbf{I}_n) \begin{bmatrix}
   \widehat{\mathbf{Q}}^T  \mathbf{X}_{k}\\
   \widehat{\mathbf{Q}}^T \mathbf{Y}_{k}\\
   \widehat{\mathbf{Q}}^T \Tilde{\mathbf{Y}}_{k}
\end{bmatrix} -  \widehat{\mathbf{Q}}^T  \mathbf{h}_k
\end{aligned}
\end{equation}
where  $\mathbf{\Theta} =  \begin{bmatrix}
    \mathbf{I}  &  \beta \tau \mathbf{I} & \mathbf{0}  \\
 \rho   \widehat{\mathbf{Q}}^T  \Tilde{\mathbf{L}} \widehat{\mathbf{Q}}   &   \rho  \widehat{\mathbf{Q}}^T  \Tilde{\mathbf{L}}  \widehat{\mathbf{Q}}   \beta \tau \mathbf{I} + \frac{1}{2}\mathbf{I} & -\frac{1}{2}\mathbf{I}  \\
  \mathbf{ 0 }&   \frac{1}{2}\mathbf{I} & \frac{1}{2}\mathbf{I}
\end{bmatrix}$.

The next step is to show that $\widehat{\mathbf{Q}}^T \Tilde{\mathbf{L}} \widehat{\mathbf{Q}}$ is negative definite by contradiction.
Let $x \in R^{N-1}$ be an arbitrary vector, since $ - \Tilde{\mathbf{L}} $ is the positive semi-definite Laplacian matrix,  the quadratic form $x^T\widehat{\mathbf{Q}}^T \Tilde{\mathbf{L}} \widehat{\mathbf{Q}}x = (\widehat{\mathbf{Q}}x )^T \Tilde{\mathbf{L}} \widehat{\mathbf{Q}}x \leq 0$. Moreover, if $ (\widehat{\mathbf{Q}}x )^T  ( \Tilde{\mathbf{A}} -\mathbf{D})\widehat{\mathbf{Q}}x =  0$, we have $\widehat{\mathbf{Q}}x = \boldsymbol{1}$.
Now, the properties of $\widehat{\mathbf{Q}}$ imply that $\widehat{\mathbf{Q}}^T\widehat{\mathbf{Q}}x = x=\widehat{\mathbf{Q}}^T \boldsymbol{1}=0$. Therefore, for all non-zero vectors $x$, the quadratic form $x^T\widehat{\mathbf{Q}}^T  \Tilde{\mathbf{L}} \widehat{\mathbf{Q}}x<0$, thus $\widehat{\mathbf{Q}}^T  \Tilde{\mathbf{L}} \widehat{\mathbf{Q}}$ is a symmetric negative-definite matrix.

We proceed now to diagonalize each block of $\mathbf{\Theta}$ with $\boldsymbol{\phi} \in \R^{(N-1)\times (N-1)}$:
$$
\begin{aligned}
\tilde{\mathbf{\Theta}} &= \boldsymbol{\phi} \mathbf{\Theta} \boldsymbol{\phi}^T=
\begin{bmatrix}
\phi & 0 & 0 \\
0 & \phi & 0 \\
0 & 0 & \phi \\
\end{bmatrix} \mathbf{\Theta} \begin{bmatrix}
\phi^T & 0 & 0 \\
0 & \phi^T & 0 \\
0 & 0 & \phi^T \\
\end{bmatrix} 
\\&=  \begin{bmatrix}
    \mathbf{I}  &  \beta \tau & 0  \\
 \rho  \phi \widehat{\mathbf{Q}}^T  \Tilde{\mathbf{L}} \widehat{\mathbf{Q}}\phi^T  &   \rho \phi\widehat{\mathbf{Q}}^T \Tilde{\mathbf{L}} \widehat{\mathbf{Q}} \phi^T   \beta \tau + \frac{1}{2}\mathbf{I} & -\frac{1}{2}\mathbf{I}  \\
   0 &   -\frac{1}{2}\mathbf{I} & \frac{1}{2}\mathbf{I}
\end{bmatrix}.
\end{aligned}
$$
We denote $ \phi\widehat{\mathbf{Q}}^T  \Tilde{\mathbf{L}} \widehat{\mathbf{Q}}\phi^T  = \operatorname{diag}\{\tilde{\lambda}_i\}_{i=2,...,N}$, where $\tilde{\lambda}_i<0$ is the eigenvalue of $  \widehat{\mathbf{Q}}^T  \Tilde{\mathbf{L}} \widehat{\mathbf{Q}}$,  $\tilde{\lambda}_{\min} = \lambda_{\min}( \widehat{\mathbf{Q}}^T  \Tilde{\mathbf{L}} \widehat{\mathbf{Q}} ) $, and $\tilde{\lambda}_{\max} = \lambda_{\max}( \widehat{\mathbf{Q}}^T  \Tilde{\mathbf{L}} \widehat{\mathbf{Q}} ) $, note that $|\tilde{\lambda}_{\max}|$ and  $|\tilde{\lambda}_{\min}|$ are the smallest nonzero eigenvalue and the largest eigenvalue of the Laplacian matrix of the graph $\mathcal{G}$, respectively. In the following, we denote $ \lambda_l = |\tilde{\lambda}_{\max}|$ and $\lambda_u = |\tilde{\lambda}_{\min}|$.
Since each block of $\tilde{\mathbf{\Theta}}$ is a diagonal matrix, there exists a permutation matrix $\mathbf{P}_0$ such that $\mathbf{P}_0 \tilde{\mathbf{\Theta}} \mathbf{P}_0^T= \mathbf{P}_0 \boldsymbol{\phi} \mathbf{\Theta} \boldsymbol{\phi}^T \mathbf{P}_0^T=
 \operatorname{blkdiag}\left\{ \mathbf{D}_i\right\}_{i=2}^N,$ where
\begin{equation}
 \mathbf{D}_i= 
 \begin{bmatrix}
    1  &  \beta \tau & 0  \\
\rho \tilde{\lambda}_i  &   \rho \tilde{\lambda}_i  \beta \tau + 0.5 & -0.5  \\
   0 &   -0.5 & 0.5
\end{bmatrix}.
\end{equation}
We diagonalize $\mathbf{D}_i=\mathbf{V}_i \mathbf{\Delta}_i \mathbf{V}_i^{-1}$, where $\mathbf{\Delta}_i$ is the diagonal matrix of $\mathbf{D}_i$'s eigenvalues, and  
\begin{equation}
 \mathbf{V}_i= 
 \begin{bmatrix}
   -\beta \tau& d_{12}& d_{13} \\
  1&   d_{22}& d_{23}  \\
   1 &1 & 1
\end{bmatrix}
\end{equation} with
$d_{12}= -\beta\tau  + ((\beta\tilde{\lambda_i}\rho\tau(\beta\tilde{\lambda_i}\rho\tau + 2))^{0.5})/(\tilde{\lambda_i}\rho)$, $d_{13} = -\beta\tau  - ((\beta\tilde{\lambda_i}\rho\tau(\beta\tilde{\lambda_i}\rho\tau + 2))^{0.5} )/(\tilde{\lambda_i}\rho)$, $d_{22} =\tilde{\lambda_i}\rho d_{12} -1$, $d_{23}= \tilde{\lambda_i}\rho d_{13} -1$.
The nonzero eigenvalues $\lambda$ of $\mathbf{D}_i$, $i=2, \ldots, N$, satisfy 
$
2 \lambda^2 + (-2 \tilde{\lambda}_i \rho \tau \beta -4) \lambda + \tilde{\lambda}_i \rho \tau \beta + 2 =0,
$
which can be written in the form:
\begin{equation} \label{root}
2\lambda^2 -2t\lambda +t =0
\end{equation}
where $t =\tilde{\lambda}_i \rho \tau \beta + 2 $.
The modulus of the roots of \eqref{root} is $1- \frac{|\tilde{\lambda}|\rho \tau \beta}{2}$ when $-2 < \tilde{\lambda} \rho \tau \beta <0$.
We conclude that we can write $\mathbf{\Theta}=(\mathbf{P}_0 \boldsymbol{\phi})^T \mathbf{V} \boldsymbol{\Delta} \mathbf{V}^{-1} (\mathbf{P}_0 \boldsymbol{\phi})$ where $\mathbf{V}=\operatorname{blkdiag}\left\{V_i\right\}_{i=2}^N $ and 
\begin{equation} \label{eq: Delta}
\boldsymbol{\Delta}=\operatorname{blkdiag}\left\{\mathbf{\Delta}_i\right\}_{i=2}^N.
\end{equation}
Moreover, $\|\mathbf{\Delta}\| = 1 - {\lambda_l\rho \tau \beta} /{2}$
when
\begin{equation} \label{eq:gamma_d_k}
\lambda_u\rho \tau \beta <2.
\end{equation}
Then, left multiplying both sides of \eqref{Q_2} by the inverse of
\begin{equation} \label{eq: v_hat}
\widehat{\mathbf{V}}=(P_0 \boldsymbol{\phi})^T\mathbf{V},
\end{equation}
which is given by $\widehat{\mathbf{V}}^{-1}=\mathbf{V}^{-1}(\mathbf{P}_0 \boldsymbol{\phi})$, yields
\begin{equation}\label{eq:delta-hat}
\widehat{\mathbf{d}}_{k+1}=\mathbf{\Delta} \widehat{\mathbf{d}}_k- \widehat{\mathbf{h}}_{k},
\end{equation}
where $
\widehat{\mathbf{d}}_k = \widehat{\mathbf{V}}^{-1}
\begin{bmatrix}
\widehat{\mathbf{Q}}^T  \mathbf{X}_{k};
   \widehat{\mathbf{Q}}^T \mathbf{Y}_{k};
   \widehat{\mathbf{Q}}^T \Tilde{\mathbf{Y}}_{k}
\end{bmatrix}
$,
$
\widehat{\mathbf{h}}_{k} =  \widehat{\mathbf{V}}^{-1} \widehat{\mathbf{Q}}^T  \mathbf{h}_k
$,
and
$
\begin{bmatrix}
\widehat{\mathbf{Q}}^T  \mathbf{X}_{k};
   \widehat{\mathbf{Q}}^T \mathbf{Y}_{k};
   \widehat{\mathbf{Q}}^T \Tilde{\mathbf{Y}}_{k}
\end{bmatrix}=\widehat{\mathbf{V}}\widehat{\mathbf{d}}_k = \boldsymbol{\phi}^T \mathbf{P}_0^T\mathbf{V}\widehat{\mathbf{d}}_k = \boldsymbol{\phi}^T \mathbf{P}_0^T\mathbf{V} \mathbf{P}_0\mathbf{P}_0^T \widehat{\mathbf{d}}_k
$.
As a consequence, from~\eqref{Q_2} it holds that
$$
\begin{aligned}
&\begin{bmatrix}
\widehat{\mathbf{Q}}^T  \mathbf{X}_{k}\\
   \widehat{\mathbf{Q}}^T \mathbf{Y}_{k}\\
   \widehat{\mathbf{Q}}^T \Tilde{\mathbf{Y}}_{k}.
\end{bmatrix}= \boldsymbol{\phi}^T \left[\begin{array}{ccc}
 -\beta \tau \mathbf{I} & d_{12}\mathbf{I} & d_{13}\mathbf{I}  \\
  \mathbf{I}&   d_{22}\mathbf{I}& d_{23} \mathbf{I} \\
   \mathbf{I} &\mathbf{I} & \mathbf{I}
\end{array}\right] \mathbf{P}_0^T  \widehat{\mathbf{d}}_k \\& = \boldsymbol{\phi}^T \left[\begin{array}{c}
 -\beta \tau \mathbf{I}  \mathbf{P}_0^T[1] + d_{12}   \mathbf{P}_0^T[2] +  d_{13}  \mathbf{P}_0^T[3]  \\
   \mathbf{P}_0^T[1] +   d_{22}  \mathbf{P}_0^T[2]+ d_{23}  \mathbf{P}_0^T[3]  \\
    \mathbf{P}_0^T[1] + \mathbf{P}_0^T[2] + \mathbf{P}_0^T[3]
\end{array}\right]   \widehat{\mathbf{d}}_k,
\end{aligned}
$$ where $\mathbf{P}_0^T[1], \mathbf{P}_0^T[2], \mathbf{P}_0^T[3]$ are the top, middle and bottom blocks of $\mathbf{P}_0^T$ respectively.
Moreover, we have  $ | d_{12} |^2 = |d_{13}|^2 \leq \frac{2\beta \tau}{\lambda_l\rho} $, $|d_{22}| = |d_{23}| =1$.
Now, if we let 
$
\beta \tau \leq \frac{2}{\lambda_l\rho},
$
and using $\|\phi\| =1$, $\|\mathbf{P}_0^T[i]\| =1$, $i = 1,2,3$, we derive that
\begin{align*}
&\| \bar{\mathbf{X}}_k- \mathbf{X}_k \|^2 = \|  \widehat{\mathbf{Q}}^T  \mathbf{X}_{k} \|^2
\\& \quad
= \| \phi^T (-\beta \tau \mathbf{I}  \mathbf{P}_0^T[1] + d_{12}   \mathbf{P}_0^T[2] +  d_{13}  \mathbf{P}_0^T[3])  \widehat{\mathbf{d}}_k  \|^2
\\& \quad \leq 3 (\beta^2 \tau^2 +  \frac{4\beta \tau}{\lambda_l\rho}  )\|  \widehat{\mathbf{d}}_k \|^2 
\leq \frac{18\beta \tau}{\lambda_l\rho} \|  \widehat{\mathbf{d}}_k\|^2.
\end{align*}
Applying the same manipulations to $\|  \bar{\mathbf{Y}}_k- \mathbf{Y}_k \|^2$, we obtain \eqref{X_Y_d} holds.
Denote now
$
\|\widehat{\mathbf{\Phi}}_k\|^2=\sum_{i=1}^N \sum_{t=0}^{\tau-1}\left\|\phi_{i, k}^t-\bar{x}_k\right\|^2 =\sum_{t=0}^{\tau-1}\left\|\Phi_{k}^t-\bar{X}_k\right\|^2. 
$
Using Assumption~\ref{as:local-costs} we derive that
\begin{align*}
&\| \sum_{t=0}^{\tau -1}( G(\mathrm{\Phi}_k^t) -  \nabla \mathrm{F}(\bar{\mathbf{X}}_k)   )\|^2
\\&
\leq 2\tau L^2 \|\widehat{\mathbf{\Phi}}_k \|^2
 +   2\tau \sum_{t=0}^{\tau -1} \| G(\mathrm{\Phi}_k^t)  - \nabla F(\mathrm{\Phi}_k^t)   \|^2.
\end{align*}
Denote $\overline{G}(\mathrm{\Phi}_k^t)=\frac{1}{N}  \sum_{i=1}^{N} g_i (\phi_{i, k}^t)$ and $\overline{\nabla \mathrm{F}}(\mathrm{\Phi}_k^t)=\frac{1}{N}  \sum_{i=1}^{N} \nabla f_i (\phi_{i, k}^t)$, we have
\begin{equation} \label{G(K)}
\begin{aligned}
&\|\sum_{t=0}^{\tau -1} \overline{G}(\mathrm{\Phi}_k^t) \|^2 
\\&= \|\frac{1}{N} \sum_i \sum_t (\nabla f_i\left(\phi_{i, k}^t\right)  +  g_i\left(\phi_{i, k}^t\right) -   \nabla f_i\left(\phi_{i, k}^t\right) ) \|^2 
\\&  \leq 2\|\sum_{t=0}^{\tau -1} \overline{\nabla \mathrm{F}}(\mathrm{\Phi}_k^t) \|^2  +  \frac{2 \tau}{N} \sum_i \sum_t\|   g_i\left(\phi_{i, k}^t\right) -   \nabla f_i\left(\phi_{i, k}^t\right) \|^2 
\end{aligned}
\end{equation}
We also have
$\|\nabla F(\bar{\mathbf{X}}_{k+1}) - \nabla F(\bar{\mathbf{X}}_{k}) \|^2 =N L^2 \left\| \bar{x}_{k+1} - \bar{x}_{k} \right\|^2 
 = N  L^2 \gamma^2 \|\sum_t \overline{G}(\mathrm{\Phi}_k^t)\|^2
$, it further holds that:
\begin{equation} \label{eq:h_k_0}
\begin{aligned}
& \|{\mathbf{h}}_{k}\|^2 
 \leq \gamma^2 ( 1+ 2\rho^2 \|\Tilde{\mathbf{L}} \|^2)  ( 2\tau L^2 \|\widehat{\mathbf{\Phi}}_k \|^2
\\&+ 2\tau \sum_{t=0}^{\tau -1} \| G(\mathrm{\Phi}_k^t)  - \nabla F(\mathrm{\Phi}_k^t)   \|^2 ) 
\\& + 6 L^2 \frac{\gamma^4}{\beta^2} \bigg(  \tau \sum_i \sum_t\left\|   g_i\left(\phi_{i, k}^t\right) -   \nabla f_i\left(\phi_{i, k}^t\right)   \right\|^2 \\&+ N \|\sum_{t=0}^{\tau -1} \overline{\nabla \mathrm{F}}(\mathrm{\Phi}_k^t) \|^2 \bigg)
\end{aligned}
\end{equation}
Recalling \eqref{eq:delta-hat}, 
using Jensen's inequality $ 
\|\widehat{\mathbf{d}}_{k+1}\|^2  \leq  \frac{1}{\|\mathbf{\Delta}\|}\|  \mathbf{\Delta}\|^2  \|\widehat{\mathbf{d}}_{k}\|^2 +
\frac{1}{1-\|\mathbf{\Delta}\|} \|\mathbf{\widehat{h}}_{k}\|^2 
 $ yields \eqref{d_k_0}.
\end{proof}

\section{Convergence analysis for \algsgd} \label{sgd_proof}
\subsection{Key bounds}

\begin{lemma} \label{lem:phi_k}
Let Assumptions \ref{as:graph}, \ref{as:local-costs}, and \ref{as:variance} hold; when $\beta <  \frac{2}{\tau\lambda_u\rho}$ and $\gamma  \leq \bar{\gamma}_1$, we have
\begin{equation} \label{phi_sgd}
\begin{aligned}
\mathbb{E} \left[ \|\widehat{\mathbf{\Phi}}_k\|^2 \right] &\leq 
\left( \frac{72\beta \tau^2}{\lambda_l\rho}  + 144\tau^3\beta^2  \right) \mathbb{E} [  \|\widehat{\mathbf{d}}_{k}\|^2]  + 4N\tau^2\gamma^2 \sigma^2
\\& + 16\tau^3 N \gamma^2 \mathbb{E} [  \| \nabla F(\bar{x}_k) \|^2 ].
\end{aligned}
\end{equation}
\end{lemma}

\smallskip

\begin{proof}
From \eqref{eq:compact-admm} we can derive that 
\begin{equation} \label{bar_x}
\begin{aligned}
\bar{x}_{k+1} - x^* &= \bar{x}_{k}  - x^*   -\frac{\gamma}{N} \sum_{t=0}^{\tau -1} \sum_{i=1}^{N} g_i(\phi_{i,k}^{t}) 
\end{aligned}
\end{equation}
and
\begin{equation}  \label{Phi_k}
      \Phi_k^{t+1} =\Phi_k^{t} + \beta \mathbf{Y}_{k}  - \gamma( G(\Phi_k^t) -  \nabla \mathrm{F}(\bar{\mathbf{X}}_k)   )
\end{equation}
Recall that by Assumption \ref{as:variance},  $\| G(\Phi_k^t) - \nabla F(\Phi_k^t)  \|^2 \leq N\sigma^2$. 
Now, suppose that $\tau \geq 2$, using Jensen's inequality we obtain 
\begin{align}
&  \mathbb{E} [ \left\|\Phi_k^{t+1}-\bar{\mathbf{X}}_k\right\|^2 ] \nonumber 
\\& =  \mathbb{E} [ \|\Phi_k^{t}-\bar{\mathbf{X}}_k  + \beta \mathbf{Y}_{k} -\gamma( \nabla \mathrm{F}(\mathrm{\Phi}_k^t) -  \nabla \mathrm{F}(\bar{\mathbf{X}}_k) ) \|^2 ] + N \gamma^2\sigma^2 \nonumber 
\\& \leq\left(1+\frac{1}{\tau-1} \right)  \mathbb{E} [ \left\|\Phi_k^{t}-\bar{\mathbf{X}}_k\right\|^2 ] + N \gamma^2\sigma^2
\\&+ \tau   \mathbb{E} [\|  \beta Y_{k} - \gamma( \nabla \mathrm{F}(\mathrm{\Phi}_k^t) -  \nabla \mathrm{F}(\bar{\mathbf{X}}_k) ) \|^2 ] \nonumber 
\\& \leq \left( 1 + \frac{1}{\tau-1} + 2 \gamma^2 \tau  L^2   \right)  \mathbb{E} [ \left\|\mathbf{\Phi}_{k}^t-\bar{\mathbf{X}}_k\right\|^2 ] \\&+2 \tau \beta^2   \mathbb{E} [ \left\|  \mathbf{Y}_{k}  \right\|^2 ]
 + \gamma^2  N\sigma^2 \nonumber  
\\& \leq\left(1+\frac{5 / 4}{\tau-1}\right)  \mathbb{E} [\|\mathbf{\Phi}_{k}^t-\bar{\mathbf{X}}_k\|^2 ] 
+  \gamma^2  N\sigma^2 +2 \tau \beta^2  \mathbb{E} [\| \mathbf{Y}_{k}  \|^2] , \label{phi_k_t_vr}
\end{align}
where the last inequality holds when 
\begin{equation} \label{eq.gamma_saga_1}
2 \gamma^2 \tau L^2 \leq \frac{1/4}{\tau -1},   
\end{equation}
which is satisfied by $\gamma\leq \bar{\gamma}_1$.
Iterating the above inequality for $t=0,..., \tau-1$ 
\begin{align*}
\mathbb{E} & [  \left\|\Phi_k^{t+1}-\bar{\mathbf{X}}_k\right\|^2] \leq \left(1+\frac{5 / 4}{\tau-1}\right) ^t \mathbb{E} [\left\|\mathbf{X}_{k}-\bar{\mathbf{X}}_k\right\|^2] +
\\& +2 \tau \beta^2 \sum_{l=0}^t  \left(1+\frac{5 / 4}{\tau-1}\right) ^l \mathbb{E} [ \left\| \mathbf{Y}_{k} -\bar{\mathbf{Y}}_k + \bar{\mathbf{Y}}_k \right\|^2 ]  \\&
+  N\gamma^2 \sigma^2 \sum_{l=0}^t  \left(1+\frac{5 / 4}{\tau-1}\right) ^l
\\&\leq 4 \mathbb{E} [ \left\|\mathbf{X}_{k}-\bar{\mathbf{X}}_k\right\|^2 ] + 4\tau N \gamma^2 \sigma^2 
+ 8\tau^2\beta^2 \mathbb{E} [ \left\| \mathbf{Y}_{k} -\bar{\mathbf{Y}}_k + \bar{\mathbf{Y}}_k  \right\|^2],
\end{align*}
where the last inequality holds by $(1+ \frac{a}{\tau -1})^t \leq \exp(\frac{at}{\tau -1})\leq \exp(a)$ for $t\leq \tau-1$ and $a = {5}/{4}$.

Summing over $t$, it follows that
\begin{equation}\label{phi_sgd_0}
\begin{split}
 \mathbb{E} [ \|\widehat{\mathbf{\Phi}}_k \|^2 ] &\leq 
4\tau \mathbb{E} [\|\mathbf{X}_{k}-\bar{\mathbf{X}}_k\|^2 ]   + 4N\tau^2\gamma^2 \sigma^2 \\& + 16\tau^3\beta^2 \mathbb{E} [\| \mathbf{Y}_{k} -\bar{\mathbf{Y}}_k \|^2 ] + 16\tau^3 N \gamma^2 \mathbb{E} [  \| \nabla F(\bar{x}_k) \|^2 ];
\end{split}
\end{equation}
moreover, it is easy to verify that \eqref{phi_sgd_0} also holds for $\tau =1$.
Using   \eqref{X_Y_d} concludes the proof.
\end{proof}

\begin{lemma} \label{lemma:d_k}
Let Assumptions \ref{as:graph}, \ref{as:local-costs}, and \ref{as:variance} hold. When $\beta <  \frac{2}{\tau\lambda_u\rho}$ and $\gamma \leq \bar{\gamma}_1$, 
\begin{equation} \label{d_k}
\begin{aligned}
&\mathbb{E} [ \|\widehat{\mathbf{d}}_{k+1}\|^2 ] 
\\&\leq ( \delta + \frac{c_0}{1-\delta}) \mathbb{E} [ \|\widehat{\mathbf{d}}_{k}\|^2 ] 
+ \frac{c_1}{1-\delta}  \mathbb{E}[ \| \sum_t \overline{\nabla F}(\boldsymbol{\Phi}^t_k) \|^2 ]
\\& + \frac{c_2}{1-\delta}\mathbb{E}[\| \nabla F(\bar{x}_k) \|^2] + \frac{c_3}{1-\delta} \sigma^2
\end{aligned} 
\end{equation}
where 
\begin{align*}
\delta & \coloneqq 1- \frac{\lambda_l\rho \tau \beta}{2}, 
\\ \beta_0 & \coloneqq \frac{72\beta \tau^2}{\lambda_l\rho}  + 144\tau^3\beta^2 
\\c_0 & \coloneqq \gamma^2 ( 1+ 2\rho^2 \|\Tilde{\mathbf{L}} \|^2)  2\tau L^2 \|\widehat{\mathbf{V}}^{-1} \|^2 \beta_0, 
\\ 
c_1 & \coloneqq \gamma^4\frac{6 L^2 }{\beta^2} N\|\widehat{\mathbf{V}}^{-1} \|^2 ,
\\c_2&  \coloneqq \gamma^4 ( 1+ 2\rho^2 \|\Tilde{\mathbf{L}} \|^2)   L^2 32\tau^4 \|\widehat{\mathbf{V}}^{-1} \|^2 ,
\\ 
c_3 & \coloneqq \gamma^4 ( 1+ 2\rho^2 \|\Tilde{\mathbf{L}} \|^2) 8 L^2 N \tau^3  \|\widehat{\mathbf{V}}^{-1} \|^2  
\\&+  \gamma^2 ( 1+ 2\rho^2 \|\Tilde{\mathbf{L}} \|^2) 2\tau^2 N\|\widehat{\mathbf{V}}^{-1} \|^2 +  6 L^2 \frac{\gamma^4}{\beta^2} N \tau^2 \|\widehat{\mathbf{V}}^{-1} \|^2,
\end{align*}
\end{lemma}

\smallskip

\begin{proof}
When   $\beta <  \frac{2}{\tau\lambda_u\rho}$ and $\gamma \leq \bar{\gamma}_1$, using \eqref{eq:h_k_0}, \eqref{phi_sgd} and Assumption \ref{as:variance}, we have 
\begin{align*}
\|\widehat{\mathbf{h}}_{k}\|^2 
 &\leq  c_0  \|\widehat{\mathbf{d}}_{k}\|^2 + c_1  \|\sum_t \overline{ \nabla F}(\boldsymbol{\Phi}^t_k) \|^2 
 \\&+ c_2\mathbb{E}[\| \nabla F(\bar{x}_k) \|^2] +  c_3 \sigma^2,
\end{align*}
together with \eqref{d_k_0}   we can then derive that \eqref{d_k} holds.
\end{proof}

\subsection{Theorem \ref{th:nonconvex_SGD}} \label{noconvex_sgd_proof}
We start our proof by recalling that the following inequality holds for all $L$-smooth function $f$, $\forall y, z \in \R^n$ \cite{nesterov2013introductory}:
\begin{equation} \label{nonconvex_inequality}
f(y) \leq f(z) + \langle \nabla f(z), y-z \rangle + (L/2) \Vert y-z  \Vert^2 
\end{equation}
Based on \eqref{bar_x},
substituting $ y = \bar{x}_{k+1}$ and  $ z = \bar{x}_{k}$ into \eqref{nonconvex_inequality}, using Assumption \ref{as:variance}, we get
\begin{align*}
&  \mathbb{E}  [ F \left(\bar{x}_{k+1}\right)  ]
\\& \leq  \mathbb{E} [F\left(\bar{x}_{k}\right)]-\gamma  \mathbb{E} [ \langle\nabla F\left(\bar{x}_{k}\right), \frac{1}{N} \sum_t \sum_i\nabla f_i\left(\phi_{i, k}^t\right)\rangle ] \\&
+\frac{\gamma^2 L}{2}  \mathbb{E} [ \|\frac{1}{N}  \sum_t \sum_i g_i(\phi_{i, k}^t)\|^2 ]\\
& \leq  \mathbb{E} [ F(\bar{x}_{k}) ]-\gamma [ \langle\nabla F\left(\bar{x}_{k}\right), \frac{1}{N} \sum_t \sum_i \nabla f_i(\phi_{i, k}^t)  ) \rangle ] \\&
+  \gamma^2 \tau L  \mathbb{E} [\sum_t  \|\frac{1}{N} \sum_i \nabla f_i\left(\phi_{i, k}^t \right)\|^2 ]+ {\gamma^2 \tau^2 L \sigma^2}.
\end{align*}
Using now $2\langle a, b\rangle=\|a\|^2+\|b\|^2-\|a-b\|^2$, we have
\begin{align*}
& - \langle\nabla F\left(\bar{x}_{k}\right), \frac{1}{N} \sum_t \sum_i \nabla f_i\left(\phi_{i, k}^t\right) \rangle 
\\& =-\frac{\tau}{2}\left\|\nabla F\left(\bar{x}_{k}\right)\right\|^2
-\frac{1}{2} \sum_t\|\frac{1}{N} \sum_i \nabla f_i\left(\phi_{i, k}^t\right)\|^2
\\& +\frac{1}{2} \sum_t\|\frac{1}{N} \sum_i \nabla f_i\left(\phi_{i, k}^t\right)-\nabla F\left(\bar{x}_{k}\right)\|^2 
\\& \leq-\frac{\tau}{2}\left\|\nabla F\left(\bar{x}_{k}\right)\right\|^2-\frac{1}{2} \sum_t\|\frac{1}{N} \sum_i \nabla f_i\left(\phi_{i, k}^t\right)\|^2 + \frac{L^2}{2 N}   \|\widehat{\mathbf{\Phi}}_k \|^2.
\end{align*}
Now, combining the two equations above and using \eqref{X_Y_d}, yields
\begin{align*}
 & \mathbb{E} [ F\left(\bar{x}_{k+1}\right) ] 
 \leq   \mathbb{E} [ F\left(\bar{x}_{k}\right) ]-\frac{\gamma \tau}{2} \mathbb{E} [ \left\|\nabla F\left(\bar{x}_{k}\right)\right\|^2 ]
 \\& -\frac{\gamma}{2}(1 - 2 \gamma \tau L) \sum_t \mathbb{E} [ \|\frac{1}{N} \sum_i \nabla f_i\left(\phi_{i, k}^t\right) \|^2 ] 
 \\& +\frac{\gamma L^2}{2 N} \mathbb{E} [ \|\widehat{\mathbf{\Phi}}_k\|^2 ] + { \gamma^2 \tau^2 L \sigma^2}.
\end{align*}
Substituting \eqref{phi_sgd} into the above inequality yields
\begin{align*}
& \mathbb{E} [ F\left(\bar{x}_{k+1}\right) ] \leq   \mathbb{E} [ F\left(\bar{x}_{k}\right) ] + \\&-\frac{\gamma\tau}{2}\left( 1 - 16L^2 \tau^2 \gamma^2 \right) \mathbb{E} [ \|\nabla F\left(\bar{x}_{k}\right) \|^2 ]
\\& -\frac{\gamma}{2}(1 - 2 \gamma L \tau) \sum_t \mathbb{E} [ \|\frac{1}{N} \sum_i \nabla f_i\left(\phi_{i, k}^t\right)\|^2 ]
\\&+ \frac{\gamma L^2}{2 N}  \left( \frac{72\beta \tau^2}{\lambda_l\rho}  + 144\tau^3\beta^2  \right) \mathbb{E} [ \|\widehat{\mathbf{d}}_k \|^2 ]
\\& + { \gamma^2 \tau^2 L \sigma^2}
  +  2 \tau^2\gamma^3\sigma^2L^2.
\end{align*}
When $\gamma \leq \min \lbrace \bar{\gamma}_2, \bar{\gamma}_3 \rbrace$, then 
\begin{equation} \label{gamma_SGD_2}
16L^2 \tau^2 \gamma^2 \leq \frac{3}{4}, \quad 2\gamma L \tau \leq \frac{3}{4},
\end{equation}
and we can upper bound the previous inequality by
\begin{align*}
& \mathbb{E} [ F\left(\bar{x}_{k+1}\right) ] \leq   \mathbb{E} [ F\left(\bar{x}_{k}\right) ]-\frac{\gamma \tau}{8} \mathbb{E} [ \|\nabla F\left(\bar{x}_{k}\right) \|^2 ] 
\\&-\frac{\gamma}{8} \sum_t \mathbb{E} [ \|\frac{1}{N} \sum_i \nabla f_i\left(\phi_{i, k}^t\right)\|^2 ] 
\\&+ \frac{\gamma L^2}{2 N}  \left( \frac{72\beta \tau^2}{\lambda_l\rho}  + 144\tau^3\beta^2  \right) \mathbb{E} [ \|\widehat{\mathbf{d}}_k \|^2 ]
\\& + { \gamma^2 \tau^2 L \sigma^2} + 2 \tau^2\gamma^3\sigma^2L^2.
\end{align*}
Rearranging the above relation, we get
\begin{align*}
&\mathcal{D}_k  
\leq \frac{8}{\gamma \tau}  \mathbb{E} \left[ \left(  \tilde{F}\left(\bar{x}_{k}\right)-  \tilde{F}\left(\bar{x}_{k+1}\right)\right) \right]\\& 
+\frac{8}{\gamma \tau} \frac{\gamma L^2}{2 N}  \left( \frac{72\beta \tau^2}{\lambda_l\rho}  + 144\tau^3\beta^2  \right)  \mathbb{E} [\|\widehat{\mathbf{d}}_k \|^2 ] 
 \\& + {8 \gamma \tau L \sigma^2} + 16 \tau\gamma^2\sigma^2L^2,
\end{align*}
where $\mathcal{D}_k$ is defined in~\eqref{eq:convergence-metric}, and $\tilde{F}\left(\bar{x}_{k}\right) = F\left(\bar{x}_{k}\right)-F(x^*)$. 

Summing over $k=0,1, \ldots, K-1$,  using $-\tilde{F}\left(\bar{x}_{k}\right) \leq 0$, it holds that
\begin{align}
& \sum_{k=0}^{K-1} \mathcal{D}_k \leq  \frac{8 \tilde{F}\left(\bar{x}^0\right)}{\gamma \tau} +  c_4 \sum_{k=0}^{K-1} \mathbb{E} [ \|\widehat{\mathbf{d}}_k\|^2 ] + Kc_5 \sigma^2 \label{sum_E}
\end{align}
where
\begin{equation} 
\begin{aligned}
c_4 &\coloneqq  \frac{ 4 L^2}{ N}  \left( \frac{72\beta \tau}{\lambda_l\rho}  + 144\tau^2\beta^2  \right) 
\\ c_5 & \coloneqq {8 \gamma \tau L} + 16 \tau\gamma^2L^2
\end{aligned}
\end{equation}

We now bound the term $\sum_{k=0}^{K-1}  \|\widehat{\mathbf{d}}_k\|^2$.
From \eqref{d_k}, we have
\begin{align}
&\mathbb{E} [ \|\widehat{\mathbf{d}}_{k+1}\|^2 ] \nonumber \\&\leq ( \delta + \frac{c_0}{1-\delta}) \mathbb{E} [ \|\widehat{\mathbf{d}}_{k}\|^2 ] 
+ \frac{c_1 \tau }{1-\delta}  \mathbb{E}[ \|\frac{1}{N} \sum_i \nabla f_i\left(\phi_{i, k}^t\right)\|^2 ]\nonumber
\\& + \frac{c_2}{1-\delta}\mathbb{E}[\| \nabla F(\bar{x}_k) \|^2] + \frac{c_3}{1-\delta} \sigma^2\nonumber
\\& \leq  \bar{\delta}
\mathbb{E} [ \| \widehat{\mathbf{d}}_k \|^2 ] + \frac{c_3}{1-\delta} \sigma^2
+ R \mathcal{D}_k   \label{d_iter}
\end{align} 
where 
\begin{equation} \label{eq:K}
\begin{split}
R &\coloneqq \operatorname{\max}\{ \frac{c_2 }{1-\delta}, \frac{c_1\tau^2}{1-\delta}\}.
\end{split}
\end{equation}
Moreover, letting  $\gamma \leq  \bar{\gamma}_4$ and $ \frac{1}{\tau\lambda_u\rho} \leq \beta <  \frac{2}{\tau\lambda_u\rho}$,  we have
\begin{equation} \label{gamma_SGD_3}
\bar{\delta} =  \delta + \frac{ c_0 }{1-\delta} <1 - \frac{\lambda_l}{4\lambda_u}.
\end{equation}
Iterating  \eqref{d_iter} now gives
$$
\mathbb{E} [ \|\widehat{\mathbf{d}}_k\|^2 ] \leq \bar{\delta}^k \mathbb{E} [ \|\widehat{\mathbf{d}}_0\|^2]+ R \sum_{\ell=0}^{k-1}\bar{\delta}^{k-1-\ell} \mathcal{D}_{\ell} + \frac{c_3 \sigma^2}{1-\bar{\delta}}
$$
and summing this inequality over $k=0, \ldots, K-1$, it follows that 
\begin{equation} \label{sum_d}
\begin{aligned}
\sum_{k=0}^{K-1}  \mathbb{E} [\|\widehat{\mathbf{d}}_k\|^2 ]
\leq \frac{\|\widehat{\mathbf{d}}_0\|^2}{1-\bar{\delta}  }+\frac{R}{ 1- \bar{\delta} } \sum_{k=0}^{K-1} \mathcal{D}_k + \frac{c_3 \sigma^2 K}{1-\bar{\delta}}.
\end{aligned}
\end{equation}

Substituting \eqref{sum_d} into \eqref{sum_E} and rearranging, we obtain
$$
\begin{aligned}
&\left(1- q_0\right)  \sum_{k=0}^{K-1} \mathcal{D}_k \leq  \frac{8 \tilde{F}\left(\bar{x}^0\right)}{\gamma \tau} + q_1 \|\hat{\mathbf{d}}_0\|^2+ Kq_2 \sigma^2,
\end{aligned}
$$
where 
\begin{equation}
    \begin{aligned}
& q_0 \coloneqq \frac{ c_4 R}{1-\bar{\delta}}, \quad q_1 \coloneqq  \frac{c_4}{1-\bar{\delta}} \quad q_2 \coloneqq\frac{c_4 c_3 }{1-\bar{\delta}} + c_5.
    \end{aligned}
\end{equation}
Since
$1- \bar{\delta} \geq \frac{\lambda_l}{4\lambda_u}$ and $1- {\delta} \geq \frac{\lambda_l}{2\lambda_u}$,  when   $\gamma \leq \min \lbrace 1, \bar{\gamma}_5, \bar{\gamma}_6 \rbrace$, we have
\begin{equation} \label{gamma_SGD_4}
q_0 \leq \frac{1}{2},
\end{equation}
and it follows that
\begin{equation} \label{eq:converge}
\frac{1}{K} \sum_{k=0}^{K-1} \mathcal{D}_k \leq \frac{16 \tilde{F}\left(\bar{x}^0\right)}{\gamma \tau K} + \frac{2q_1}{K} \|\hat{\mathbf{d}}_0\|^2+ 2q_{2}\sigma^2.
\end{equation}
By collecting all step-size conditions,  if the step-size $\gamma < \bar\gamma_{\text{sgd}}\coloneqq \min _{i=1, 2,  \ldots, 6} \bar\gamma_i$,  then  \eqref{eq:converge}  holds, the states $\{\mathbf{X}_k\}$ generated by {\algsgd} converge to the neighborhood of the stationary point, concluding the proof.

\section{Convergence analysis for {\algsaga}}
\subsection{Key bounds}
We start by deriving an upper bound for the variance of the gradient estimator
$\mathbb{E}[\| g_i(\phi_{i, k}^t) -\nabla f_i({\phi}_{i,k}^t)\|^2]$. 
Define $t_i^k$ as the averaged consensus gap of the auxiliary variables of $\{r_{i, h, k}^{k}\}_{h=1}^{m_i}$ at node $i$:
\begin{align*}
&t_{i,k}^{t}=\frac{1}{m_i} \sum_{h=1}^{m_i}\|r_{i, h,k}^{t}-\bar{x}_k\|^2, 
\\&
t^{t}_k=\sum_{i=1}^N t_{i,k}^{t} = \frac{1}{m_i} \sum_{h=1}^{m_i}\|\mathbf{r}_{h,k}^{t}-\bar{\mathbf{X}}_k\|^2, 
\\&{t}_k=\sum_{t=0}^{\tau-1}t^{t}_k =\sum_{t=0}^{\tau -1 } \sum_{i=1}^N t_{i,k}^{t}.
\end{align*}

By the updates of $g_i(\phi_{i, k}^t)$ in {\algsaga},
\begin{align*}
& \mathbb{E}\left[\|g_i(\phi_{i, k}^t) - \nabla f_i\left( \phi_{i, k}^t \right)\|^2 \right] \\
& =\mathbb{E}[
\|  \frac{1}{|\mathcal{B}_i|} \sum_{h \in \mathcal{B}_i} \nabla f_{i, h}\left( \phi_{i, k}^t \right)- \frac{1}{|\mathcal{B}_i|} \sum_{h \in \mathcal{B}_i} \nabla f_{i, h}\left(r_{i, h, k}^{t}\right) \\&
 - (\nabla f_i\left(\phi_{i, k}^t\right)-\frac{1}{m_i} \sum_{h=1}^{m_j} \nabla f_{i, h}(r_{i, h,k}^{k})) \|^2 
] 
\\& \leq \mathbb{E}\left[\| \frac{1}{|\mathcal{B}_i|} \sum_{h \in \mathcal{B}_i} \nabla f_{i, h}\left( \phi_{i, k}^t \right)- \frac{1}{|\mathcal{B}_i|} \sum_{h \in \mathcal{B}_i} \nabla f_{i, h}\left(r_{i, h, k}^{t}\right)\|^2 \right] 
\\& \leq \mathbb{E}\left[ \frac{1}{|\mathcal{B}_i|}  \sum_{h \in \mathcal{B}_i} \left\|  \nabla f_{i, h}\left( \phi_{i, k}^t \right)-  \nabla f_{i, h}\left(r_{i, h, k}^{t}\right)\right\|^2 \right] 
\\&
= \frac{1}{|\mathcal{B}_i|}  \sum_{h \in \mathcal{B}_i}\mathbb{E}\left[  \left\|  \nabla f_{i, h}\left( \phi_{i, k}^t \right)-  \nabla f_{i, h}\left(r_{i, h, k}^{t}\right)\right\|^2 \right] 
\\
& \leq 2 L^2\left\|\phi_{i, k}^t -\overline{{x}}_k\right\|^2+2 L^2 \mathbb{E} [t_{i,k}^{t}], 
\end{align*}
where in the first inequality we use $ \mathbb{E} [ \|a- \mathbb{E}[a] \|^2 ] \leq \mathbb{E} [ \| a \|^2 ]$ with $ a = \nabla f_{i, h}( \phi_{i, k}^t )-\nabla f_{i, h}(r_{i, h, k}^{t}) $; and in the second inequality we use the smoothness of the costs.
As a consequence, we have
\begin{align}\label{G_VR}
  &\mathbb{E} [ \sum_i \sum_t \|   g_i\left(\phi_{i, k}^t\right) -   \nabla f_i(\phi_{i, k}^t) \|^2  ] \leq 2L^2 \|\widehat{\mathbf{\Phi}}_k \|^2 + 2 L^2 \mathbb{E}[{t}_k].
\end{align}

\begin{lemma} \label{lem:phi_k_vr}
Let Assumptions \ref{as:graph} and \ref{as:local-costs} hold; when $\beta \tau \leq \frac{2}{\lambda_u\rho}$ and $\gamma \leq \bar{\gamma}_7$, we have
\begin{equation}  \label{phi_vr}
\begin{aligned}
&\mathbb{E} [\|\widehat{\mathbf{\Phi}}_k\|^2 ]
\leq \left( \frac{72\beta \tau^2}{\lambda_l\rho}  + 144\tau^3\beta^2\right) \mathbb{E} [\| \widehat{\mathbf{d}}_k \|^2 ] \\& + 16 \tau^3\gamma^2 N  \mathbb{E} [\left\| \nabla F(\bar{x}_k)\right\|^2]  + 32 \tau^2 \gamma^2  L^2  \mathbb{E} [t_k].
\end{aligned}
\end{equation}
\end{lemma}
\begin{proof}
Suppose that $\tau \geq 2$, using \eqref{Phi_k} and \eqref{G_VR} we have
\begin{equation}
\begin{aligned}
&  \mathbb{E} [ \left\|\Phi_k^{t+1}-\bar{\mathbf{X}}_k\right\|^2 ] 
\\& =  \mathbb{E} [ \|\Phi_k^{t}-\bar{\mathbf{X}}_k  + \beta \mathbf{Y}_{k} -\gamma( G(\mathrm{\Phi}_k^t) -  \nabla \mathrm{F}(\bar{\mathbf{X}}_k) ) \|^2 ]  
\\& \leq\left(1+\frac{1}{\tau-1} \right)  \mathbb{E} [ \left\|\Phi_k^{t}-\bar{\mathbf{X}}_k\right\|^2 ] 
\\&+ \tau   \mathbb{E} [\|  \beta \mathbf{Y}_{k} - \gamma( G(\mathrm{\Phi}_k^t) -  \nabla \mathrm{F}(\bar{\mathbf{X}}_k) ) \|^2 ] 
\\& \leq \left( 1 + \frac{1}{\tau-1} + 4 \gamma^2 \tau( 2L^2+ L^2  ) \right)  \mathbb{E} [ \left\|\mathbf{\Phi}_{k}^t-\bar{\mathbf{X}}_k\right\|^2 ] \\&+2 \tau \beta^2   \mathbb{E} [ \left\|  \mathbf{Y}_{k}  \right\|^2 ]
 + 4\tau \gamma^2 
( 2  L^2  \mathbb{E} [ t^{t}_k ] )  
\\& \leq\left(1+\frac{5 / 4}{\tau-1}\right)  \mathbb{E} [\|\mathbf{\Phi}_{k}^t-\bar{\mathbf{X}}_k\|^2 ] 
\\&+ 8 \tau \gamma^2  L^2  \mathbb{E} [ t^{t}_k] +2 \tau \beta^2  \mathbb{E} [\| \mathbf{Y}_{k}  \|^2] , \label{phi_k_t_vr}
\end{aligned}
\end{equation}
where the last inequality holds when 
\begin{equation} \label{eq.gamma_saga_1}
4 \gamma^2 \tau(2L^2+ L^2) \leq \frac{1/4}{\tau -1}, 
\end{equation}
which can be satisfied when $\gamma \leq \bar{\gamma}_7$.
Similar to Lemma~\ref{lem:phi_k}, we can derive that~\eqref{phi_vr} holds,
which concludes the proof.
\end{proof}

The following lemma provides the bound on $t_k$.
\begin{lemma} \label{lemma:tk}
Let $\left\{t_k\right\}$ be the iterates generated by {\algsaga}. If  $\beta \tau \leq \frac{2}{\lambda_l\rho}$ and  $\gamma \leq \min \lbrace \bar{\gamma}_8, \bar{\gamma}_9,  \bar{\gamma}_{10}\rbrace$, we have for all $k \in \N$:
\begin{equation}\label{t_vr}
\begin{aligned}
& \mathbb{E}[t_{k}]\leq 2(s_0+s_1 )\mathbb{E} [\| \widehat{\mathbf{d}}_k\|^2] + 2s_2 \mathbb{E} [\| \nabla F(\bar{x}_k)\|^2 ],
\end{aligned}
\end{equation}
where
\begin{align}
&  s_0 = \frac{36\beta \tau^2 m_u }{\lambda_l\rho} +\frac{144 \tau^2 m_u}{m_l}\beta^2 \nonumber
\\& s_1 = \left( \frac{72\beta \tau^2}{\lambda_l\rho}  +144\tau^3\beta^2 \right) \frac{8m_u \tau}{m_l} \nonumber
\\& s_2 = \frac{16N \gamma^2 m_u \tau^2}{m_l} + \frac{8m_u \tau}{m_l}16 \tau^3\gamma^2 N. \label{eq.s}
\end{align}
\end{lemma}

\smallskip

\begin{proof}
From Algorithm~\ref{alg:lt-saga-admm}, 
$\forall k,  r_{i, h, k}^{ t+1} = r_{i, h,k}^{k}$ with probability $1-\frac{1}{m_i}$ and $r_{i, h,k}^{t+1} =  \phi_{i,k}^{t+1}$ with probability $\frac{1}{m_i}$, therefore, 
\begin{align*}
& \mathbb{E}[t^{t+1}_k] 
 = \frac{1}{m_i} \sum_{h=1}^{m_i} \mathbb{E}[ \|\mathbf{r}_{h,k}^{t+1}-\bar{\mathbf{X}}_{k}\|^2]
\\& =\frac{1}{m_i} \sum_{h=1}^{m_i} \mathbb{E}[ (1-\frac{1}{m_i})\|\mathbf{r}_{h,k}^{t}-\bar{\mathbf{X}}_{k} \|^2  +\frac{1}{m_i}\|\Phi_{k}^{t+1}-\bar{\mathbf{X}}_{k} \|^2]
\\
& = \left(1-\frac{1}{m_i}\right) \frac{1}{m_i} \sum_{h=1}^{m_i}\mathbb{E}\left[ \|\mathbf{r}_{h,k}^{t}-\bar{\mathbf{X}}_{k} \|^2 \right] \\ &+\frac{1}{m_i} \mathbb{E} [\|\Phi_{k}^{t+1}-\bar{\mathbf{X}}_{k} \|^2].
\end{align*}
Denote
$q^t_k = \beta \mathbf{Y}_{k} - \gamma( G(\mathrm{\Phi}_k^t) -  \nabla \mathrm{F}(\bar{\mathbf{X}}_k) ),$
we have
$\left\|\mathrm{\Phi}_{k}^{t+1}-\bar{\mathbf{X}}_{k}\right\|^2 = \left\| \mathrm{\Phi}_{k}^{t+1}-  \mathrm{\Phi}_{k}^{t} + \mathrm{\Phi}_{k}^{t} -\bar{\mathbf{X}}_{k}\right\|^2 
\leq 2 \| \Phi_{k}^{t} -\bar{\mathbf{X}}_{k} \|^2  + 2\| q^t_k\|^2,
$ and
\begin{align*}
&\mathbb{E} [\| q^t_k \|^2] 
\\& \leq 2 \gamma^2  \mathbb{E} [ \|    G(\mathrm{\Phi}_k^t) -  \nabla \mathrm{F}(\bar{\mathbf{X}}_k) \|^2 ] + 2\beta^2   \mathbb{E} [ \left\| \mathbf{Y}_{k}  \right\|^2 ] \nonumber  
\\& \leq4 \gamma^2( 2L^2+ L^2  )  \mathbb{E} [ \left\|\mathrm{\Phi}_{k}^t-\bar{\mathbf{X}}_k\right\|^2 ] +2 \beta^2   \mathbb{E} [ \left\|  \mathbf{Y}_{k}  \right\|^2 ]
\\& + 4\gamma^2 
( 2  L^2  \mathbb{E} [ t^{t}_k ] )
\\& \leq 12 \gamma^2 L^2 \left\|\mathrm{\Phi}_{k}^t-\bar{\mathbf{X}}_k\right\|^2 + 4\gamma^2 N \| \nabla {F}(\bar{{x}}_k)\|^2 
\\&+ 4\beta^2 \|  \mathbf{Y}_{k}  -  \mathbf{\bar{Y}}_{k}  \|^2 + 8 \gamma^2 L^2  \mathbb{E} [ t^{t}_k ],
\end{align*}
it follows that
\begin{align}
& \mathbb{E}[t^{t+1}_k] 
= (1-\frac{1}{m_i}) \frac{1}{m_i} \sum_{h=1}^{m_i}\|\mathbf{r}_{h,k}^{t}-\bar{\mathbf{X}}_{k}\|^2 
\\&+\frac{1}{m_i}\|\mathrm{\Phi}_{k}^{t+1}-\bar{\mathbf{X}}_{k}\|^2 \nonumber
\\& \leq (1-\frac{1}{m_i})t^{t}_k +\frac{1}{m_i} 
\left(2 \| \mathrm{\Phi}_{k}^{t} -\bar{\mathbf{X}}_{k} \|^2  + 2\| q^t_k\|^2\right) \nonumber
\\& \leq \left(1-\frac{1}{m_u} + \frac{16 \gamma^2  L^2}{m_l}  \right)\mathbb{E}[t^{t}_k] 
\\&+ \left(  \frac{2}{m_l} +  \frac{24\gamma^2L^2}{m_l}  \right)\mathbb{E} [  \|\mathrm{\Phi}_{k}^t-\bar{\mathbf{X}}_k\|^2  \nonumber
\\& + \frac{72}{m_l}\beta^2\| \widehat{\mathbf{d}}_k \|^2 + \frac{8N}{m_l}\gamma^2\| \nabla F(\bar{\mathbf{X}}_k) \|^2 \nonumber
\\& \leq \left(1-\frac{1}{2m_u}  \right)\mathbb{E}[t^{t}_k] + \frac{4}{m_l} \mathbb{E} [  \|\mathrm{\Phi}_{k}^t-\bar{\mathbf{X}}_k\|^2 
\\&+ \frac{72}{m_l}\beta^2\| \widehat{\mathbf{d}}_k \|^2 + \frac{8N}{m_l}\gamma^2\| \nabla F(\bar{\mathbf{X}}_k) \|^2 \label{t_k_update}
\end{align}
where the last inequality holds when 
\begin{equation} \label{eq.gamma_saga_2_0}
\frac{16 \gamma^2  L^2}{m_l} < \frac{1}{2m_u},
    \quad
   {24\gamma^2L^2}< 2. 
\end{equation}
Iterating \eqref{t_k_update} for $t=0,...,\tau-1$ then yields:
\begin{align*}
& \mathbb{E} [  t^{t}_k ]
\leq(1-\frac{1}{2m_u})^t \mathbb{E} [ \|\mathbf{X}_{k}-\bar{\mathbf{X}}_k \|^2]
\\& +\frac{72}{m_l}\beta^2 \sum_{l=0}^{t-1}  (1-\frac{1}{2m_u})^{t-1-l} \mathbb{E}[ \| \widehat{\mathbf{d}}_k \|^2] 
\\&+ \frac{8N \gamma^2}{m_l} \sum_{l=0}^{t-1}  (1-\frac{1}{2m_u})^l  \| \nabla F(\bar{x}_k) \|^2
 \\&+  \frac{4}{m_l}\sum_{l=0}^{t-1}  (1-\frac{1}{2m_u})^{t-1-l} \mathbb{E} [  \|\Phi_{k}^{l}-\bar{\mathbf{X}}_k\|^2
 \\&\leq \frac{36\beta \tau m_u }{\lambda_l\rho} \mathbb{E} [\| \widehat{\mathbf{d}}_k\|^2]
+ \frac{16N \gamma^2 m_u \tau}{m_l}  \| \nabla F(\bar{x}_k) \|^2
 \\&+  \frac{8m_u \tau}{m_l}\sum_{l=0}^{t-1}  \mathbb{E} [  \| \mathrm{\Phi}_{k}^{l}-\bar{\mathbf{X}}_k\|^2
 +\frac{144 m_u \tau}{m_l}\beta^2 \mathbb{E}[ \| \widehat{\mathbf{d}}_k \|^2].
\end{align*}
Summing the above relation over $t = 0, 1,..., \tau -1$ we get: 
\begin{align*}
&\mathbb{E} [  t_{k} ] \\&  \leq 
 \left( \frac{36\beta \tau^2 m_u }{\lambda_l\rho} +\frac{144 \tau^2 m_u}{m_l}\beta^2\right)
\mathbb{E} [\| \widehat{\mathbf{d}}_k\|^2] + \frac{8m_u \tau}{m_l} \|\widehat{\mathbf{\Phi}}_k \|^2 
\\&+ \frac{16N \gamma^2 m_u \tau^2}{m_l} \| \nabla F(\bar{x}_k) \|^2,
\end{align*}
and using \eqref{phi_vr} then yields
\begin{align*}
\mathbb{E} [  t_{k} ] &\leq (s_0+s_1 )\mathbb{E} [\| \widehat{\mathbf{d}}_k\|^2] + s_2 \mathbb{E} [\| \nabla F(\bar{x}_k)\|^2 ] \\&+  \frac{8m_u \tau}{m_l} 32 \tau^2 \gamma^2  L^2 \mathbb{E} [t_k],
\end{align*}
where $s_0$, $s_1$ and $s_2$ are defined in \eqref{eq.s}.
Letting 
\begin{equation}  \label{eq.gamma_saga_2_1}
 \frac{8m_u \tau}{m_l} 32 \tau^2 \gamma^2  L^2  < \frac{1}{2}, 
\end{equation}
and thus \eqref{t_vr} holds. The conditions \eqref{eq.gamma_saga_2_0} and \eqref{eq.gamma_saga_2_1} hold when $\gamma \leq \min \lbrace \bar{\gamma}_8, \bar{\gamma}_9,  \bar{\gamma}_{10}\rbrace$.
\end{proof}

\smallskip

\begin{lemma} \label{lem:vr_d_k}
Let Assumptions \ref{as:graph} and \ref{as:local-costs} hold; when  $\beta \tau \leq \frac{2}{\lambda_u\rho}$ and   $\gamma < \min \lbrace \bar{\gamma}_1,  \bar{\gamma}_{7} \rbrace$,
it holds that $\forall k \geq 0$
\begin{equation} \label{d_k_vr}
\begin{aligned}
& \mathbb{E} [ \|\widehat{\mathbf{d}}_{k+1}\|^2 ]
\leq 
( \delta + \frac{ \tilde{q}_0}{1-\delta})
 \mathbb{E} [ \|\widehat{\mathbf{d}}_{k}\|^2 ]
\\&+
\frac{\tilde{q}_1}{1-\delta}  \mathbb{E} [ \| \sum_t \overline{\nabla F}(\boldsymbol{\Phi}^t_k) \|^2]
+ \frac{\tilde{q}_2}{1-\delta}  \mathbb{E} [\| \nabla F(\bar{x}_k) \|^2 ],
\end{aligned} 
\end{equation}
where 
\begin{align*} 
 \tilde{\beta}_0 &\coloneqq \frac{72\beta \tau^2}{\lambda_l\rho}  + 144\tau^3\beta^2,
\\\tilde{c}_0&  \coloneqq  \gamma^2 ( 1+ 2\rho^2 \|\Tilde{\mathbf{L}} \|^2) 6\tau L^2 \|\widehat{\mathbf{V}}^{-1} \|^2 + 6 L^2 \frac{\gamma^4}{\beta^2}  2\tau L^2\|\widehat{\mathbf{V}}^{-1} \|^2 , 
\\\tilde{c}_1 &  \coloneqq\gamma^2 ( 1+ 2\rho^2 \|\Tilde{\mathbf{L}} \|^2) 4\tau L^2 \|\widehat{\mathbf{V}}^{-1} \|^2 + 6 L^2 \frac{\gamma^4}{\beta^2}  2\tau L^2\|\widehat{\mathbf{V}}^{-1} \|^2 , 
\\\tilde{c}_2 & \coloneqq 6 L^2 \frac{\gamma^4}{\beta^2}   N\|\widehat{\mathbf{V}}^{-1} \|^2, \\
 \tilde{q}_0 &\coloneqq \tilde{c}_0  \tilde{\beta}_0 + 2(s_0 +s_1) (\tilde{c}_1+ 32 \tau^2 \gamma^2 L^2 \tilde{c}_0),
\\ \tilde{q}_1 & \coloneqq   6 L^2 \frac{\gamma^4}{\beta^2} N, 
\\
 \tilde{q}_2& \coloneqq  16\tau^3\gamma^2N \tilde{c}_0 + 2s_2(\tilde{c}_1+ 32 \tau^2 \gamma^2 L^2 \tilde{c}_0).
\end{align*}
\end{lemma}
\begin{proof}
When  $\beta \tau \leq \frac{2}{\lambda_u\rho}$ and  $\gamma < \min \lbrace \bar{\gamma}_1,  \bar{\gamma}_{7} \rbrace$,
substituting \eqref{G_VR} and \eqref{phi_vr} into \eqref{eq:h_k_0} then yields 
\begin{equation*} 
\begin{aligned}
& \|{\mathbf{h}}_{k}\|^2 
 \leq \gamma^2 ( 1+ 2\rho^2 \|\Tilde{\mathbf{L}} \|^2)  ( 6\tau L^2 \|\widehat{\mathbf{\Phi}}_k \|^2 + 4\tau L^2 \mathbb{E}[{t}_k] ) 
\\& + 6 L^2 \frac{\gamma^4}{\beta^2} \bigg( 4\tau L^2 \|\widehat{\mathbf{\Phi}}_k \|^2 + 4\tau L^2 \mathbb{E}[{t}_k]+ N \|\sum_{t=0}^{\tau -1} \overline{\nabla \mathrm{F}}(\mathrm{\Phi}_k^t) \|^2 \bigg)
\end{aligned}
\end{equation*}
and
\begin{align*}
& \|\widehat{\mathbf{h}}_{k}\|^2 
\leq \tilde{c}_0 \|\widehat{\mathbf{\Phi}}_k\|^2 + \tilde{c}_1 t_k + \tilde{c}_2  \|\sum_t \overline{ \nabla F}(\boldsymbol{\Phi}^t_k) \|^2
\\& \leq \tilde{q}_0 \mathbb{E} [ \|\widehat{\mathbf{d}}_{k}\|^2 ] +    \tilde{q}_1  \mathbb{E} [ \| \sum_t \overline{\nabla F}(\boldsymbol{\Phi}^t_k) \|^2] +  \tilde{q}_2   \mathbb{E} [\| \nabla F(\bar{x}_k) \|^2 ],
\end{align*}
 together with \eqref{d_k_0}, it proves that \eqref{d_k_vr} holds.
\end{proof}

\subsection{Theorem \ref{th:nonconvex_SGD_VR}} \label{noconvex_sgd_vr_proof}
Based on \eqref{bar_x},
substituting $ y = \bar{x}_{k+1}$ and  $ z = \bar{x}_{k}$ into \eqref{nonconvex_inequality} and using \eqref{G_VR}, we get
\begin{align*}
&  \mathbb{E}  [ F \left(\bar{x}_{k+1}\right)  ]
\\& \leq  \mathbb{E} [F\left(\bar{x}_{k}\right)]-\gamma  \mathbb{E} [ \langle\nabla F\left(\bar{x}_{k}\right), \frac{1}{N} \sum_t \sum_i\nabla f_i\left(\phi_{i, k}^t\right)\rangle ] \\&
+\frac{\gamma^2 L}{2}  \mathbb{E} [ \|\frac{1}{N}  \sum_t \sum_i g_i(\phi_{i, k}^t)\|^2 ]\\
& \leq  \mathbb{E} [ F(\bar{x}_{k}) ]-\gamma [ \langle\nabla F\left(\bar{x}_{k}\right), \frac{1}{N} \sum_t \sum_i \nabla f_i(\phi_{i, k}^t)  ) \rangle ] \\&
+  \gamma^2 \tau L  \mathbb{E} [\sum_t  \|\frac{1}{N} \sum_i \nabla f_i\left(\phi_{i, k}^t \right)\|^2 ]+ 2\gamma^2 \tau L^3 ( \|\widehat{\mathbf{\Phi}}_k \|^2 +  \mathbb{E}[{t}_k]).
\end{align*}
Using now $2\langle a, b\rangle=\|a\|^2+\|b\|^2-\|a-b\|^2$, we have
\begin{align*}
& - \langle\nabla F\left(\bar{x}_{k}\right), \frac{1}{N} \sum_t \sum_i \nabla f_i\left(\phi_{i, k}^t\right) \rangle 
\\& =-\frac{\tau}{2}\left\|\nabla F\left(\bar{x}_{k}\right)\right\|^2
-\frac{1}{2} \sum_t\|\frac{1}{N} \sum_i \nabla f_i\left(\phi_{i, k}^t\right)\|^2
\\& +\frac{1}{2} \sum_t\|\frac{1}{N} \sum_i \nabla f_i\left(\phi_{i, k}^t\right)-\nabla F\left(\bar{x}_{k}\right)\|^2 
\\& \leq-\frac{\tau}{2}\left\|\nabla F\left(\bar{x}_{k}\right)\right\|^2-\frac{1}{2} \sum_t\|\frac{1}{N} \sum_i \nabla f_i\left(\phi_{i, k}^t\right)\|^2 + \frac{L^2}{2 N}   \|\widehat{\mathbf{\Phi}}_k \|^2.
\end{align*}
Now, combining the two equations above and using \eqref{X_Y_d}, yields
\begin{align*}
 & \mathbb{E} [ F\left(\bar{x}_{k+1}\right) ] 
 \leq   \mathbb{E} [ F\left(\bar{x}_{k}\right) ]-\frac{\gamma \tau}{2} \mathbb{E} [ \left\|\nabla F\left(\bar{x}_{k}\right)\right\|^2 ]
 \\& -\frac{\gamma}{2}(1 - 2 \gamma \tau L) \sum_t \mathbb{E} [ \|\frac{1}{N} \sum_i \nabla f_i\left(\phi_{i, k}^t\right) \|^2 ] 
 \\& +\frac{\gamma L^2}{2 N} \mathbb{E} [ \|\widehat{\mathbf{\Phi}}_k\|^2 ] +2\gamma^2 \tau L^3 ( \|\widehat{\mathbf{\Phi}}_k \|^2 +  \mathbb{E}[{t}_k]).
\end{align*}
Using \eqref{phi_vr} and \eqref{t_vr} we have
\begin{align*}
& \mathbb{E}[ F\left(\bar{x}_{k+1}\right) ]\leq  \mathbb{E}[ F\left(\bar{x}_{k}\right) ]-\frac{\gamma \tau}{2}  \mathbb{E}[\left\|\nabla F\left(\bar{x}_{k}\right)\right\|^2 ]\\& -\frac{\gamma}{2}(1-2 \gamma \tau L) \mathbb{E}[\sum_t  \|\frac{1}{N} \sum_i \nabla f_i\left(\phi_{i, k}^t\right) \|^2] 
\\&+ \tilde{q}_3 \mathbb{E}[\| \nabla F(x) \|^2] + \tilde{q}_4 \mathbb{E}[\| \widehat{\mathbf{d}}_k \|^2], 
\end{align*}
where 
\begin{align*}
\tilde{q}_3 &\coloneqq 16\tau^3\gamma^2N (\frac{\gamma L^2}{2 N}+ 2\gamma^2 \tau L^3) \\&+ 2s_2(2\gamma^2 \tau L^3+ 32 \tau^2 \gamma^2 L^2 (\frac{\gamma L^2}{2 N}+ 2\gamma^2 \tau L^3)), 
\\ \tilde{q}_4 & \coloneqq (\frac{\gamma L^2}{2 N}+ 2\gamma^2 \tau L^3)  \tilde{\beta}_0 \\&+ 2(s_0 +s_1) (2\gamma^2 \tau L^3+ 32 \tau^2 \gamma^2 L^2 (\frac{\gamma L^2}{2 N}+ 2\gamma^2 \tau L^3)).
\end{align*}
Letting $\gamma \leq \min \lbrace 1, \bar{\gamma}_{11}, \bar{\gamma}_{12}\rbrace$, then
\begin{equation} \label{eq.gamma_saga_4}
\begin{aligned}
&  \tilde{q}_3 \leq \frac{3\gamma \tau}{8}, \quad 2\gamma \tau L \leq  \frac{3}{4},
\end{aligned}
\end{equation}
and we can upper bound the previous inequality by
\begin{align*}
&  \mathbb{E}[F\left(\bar{x}_{k+1}\right) ] \leq   \mathbb{E}[ F\left(\bar{x}_{k}\right) ]-\frac{\gamma \tau}{8}  \mathbb{E}[\|\nabla F\left(\bar{x}_{k}\right) \|^2 ]\\& -\frac{\gamma}{8} \mathbb{E}[\sum_t\|\frac{1}{N} \sum_i \nabla f_i\left(\phi_{i, k}^t\right)\|^2 ] +\tilde{q}_4 \mathbb{E}[\|  \widehat{\mathbf{d}}_k  \|^2]. 
\end{align*}
Rearranging we get
\begin{equation} \label{final_E}
\begin{aligned}
\mathcal{D}_k  
&\leq \frac{8}{\gamma \tau}( \mathbb{E}[ \tilde{F}\left(\bar{x}_{k}\right)]- \mathbb{E}[ \tilde{F}\left(\bar{x}_{k+1}\right)])+ \frac{8}{\gamma \tau} 
\tilde{q}_4 \mathbb{E}[\|  \widehat{\mathbf{d}}_k  \|^2],
\end{aligned}
\end{equation}
where 
$\mathcal{D}_k $ is defined in~\eqref{eq:convergence-metric}, and $\tilde{F}\left(\bar{x}_{k}\right) = F\left(\bar{x}_{k}\right)-F(x^*)$. 
Summing  \eqref{final_E} over $k=0,1, \ldots, K-1$ and using $-\tilde{F}\left(\bar{x}_{k}\right) \leq 0$, it holds that
\begin{align} \label{E_1}
\sum_{r=0}^{K-1} \mathcal{D}_k \leq \frac{8 \tilde{F}\left(\bar{x}_0\right)}{\gamma \tau} 
+ \frac{8\tilde{q}_4}{\gamma \tau}   \sum_{k=0}^{K-1}  \|\mathbb{E}[\hat{\mathbf{d}}_k\|^2]. 
\end{align}

According to \eqref{d_k_vr},
 we derive that $\forall k\geq 0$,
\begin{equation} \label{iter_G}
\begin{aligned}
&\mathbb{E} [ \|\widehat{\mathbf{d}}_{k+1}\|^2 ]
\leq 
( \delta + \frac{ \tilde{q}_0}{1-\delta})
 \mathbb{E} [ \|\widehat{\mathbf{d}}_{k}\|^2 ]
\\&+
\frac{\tilde{q}_1}{1-\delta}  \mathbb{E} [ \| \sum_t \overline{\nabla F}(\boldsymbol{\Phi}^t_k) \|^2]
+ \frac{\tilde{q}_2}{1-\delta}  \mathbb{E} [\| \nabla F(\bar{x}_k) \|^2 ]
\\& \leq  \tilde{\delta} \mathbb{E} [ \|\widehat{\mathbf{d}}_{k}\|^2 ] + \tilde{R} \mathcal{D}_k,
\end{aligned} 
\end{equation}
where 
$$ \tilde{R} = \max \left\{\frac{\tilde{q}_1 \tau^2}{1-\delta} ,  \frac{\tilde{q}_2}{1-\delta} \right\}.$$
Letting  $\gamma \leq \min \{ \bar{\gamma}_1, \bar{\gamma}_{13} \}$ and $ \frac{1}{\tau\lambda_u\rho} \leq \beta <  \frac{2}{\tau\lambda_u\rho}$, then
\begin{equation} \label{eq.gamma_saga_5}
\tilde{\delta} =   \delta + \frac{ \tilde{q}_0}{1-\delta} \leq 1 - \frac{\lambda_l}{4\lambda_u}.
\end{equation}
Iterating \eqref{iter_G} yields $\forall k\geq1$,
$
\mathbb{E}[\|\widehat{\mathbf{d}}_{k}\|^2] \leq \tilde{\delta}^k  \mathbb{E}[\|\widehat{\mathbf{d}}_{0}\|^2]  +  \tilde{R} \sum_{\ell=0}^{k-1}\tilde{\delta}^{k-1-\ell}  \mathcal{D}_{\ell},
$
and summing over $k=0, \ldots, K-1$ it holds that 
\begin{equation} \label{d_1}
\begin{aligned}
& \sum_{k=0}^{K-1} \mathbb{E}[\|\widehat{\mathbf{d}}_{k}\|^2] \leq \frac{1}{1-\tilde{\delta}}  \|\widehat{\mathbf{d}}_{0}\|^2 
+\sum_{k=0}^{K-1}\frac{\tilde{R}}{1- \tilde{\delta}}\mathcal{D}_{k}.
\end{aligned}
\end{equation}

Substituting \eqref{d_1} into \eqref{E_1}, and rearranging, we obtain

\begin{align} 
& \sum_{r=0}^{K-1} \mathcal{D}_k \leq \frac{8 \tilde{F}\left(\bar{x}_0\right)}{\gamma \tau} 
+ \frac{8\tilde{q}_4}{\gamma \tau}   \sum_{k=0}^{K-1}  \|\mathbb{E}[\hat{\mathbf{d}}_k\|^2]. 
\end{align}

\begin{align*}
&(1-  \frac{\tilde{R}}{1- \tilde{\delta}} \frac{8\tilde{q}_4}{\gamma \tau}) \sum_{k=0}^{K-1} \mathcal{D}_k 
 \leq  \frac{8 \tilde{F}\left(\bar{x}_0\right)}{\gamma \tau } 
 + \frac{8\tilde{q}_4 }{\gamma \tau (1- \tilde{\delta})}  \|\widehat{\mathbf{d}}_{0}\|^2.
\end{align*}
Since 
$
1- \tilde{\delta}  \geq \frac{\lambda_l}{4\lambda_u},
$
let  $\gamma \leq \min \lbrace \bar{\gamma}_{14}, \bar{\gamma}_{15}\rbrace$, then
\begin{equation} \label{eq.gamma_saga_6}
\frac{\tilde{R}}{1- \tilde{\delta}} \frac{8 \tilde{q}_4}{\gamma \tau} 
\leq \frac{1}{2}, 
\end{equation}
and therefore it follows that
\begin{equation} \label{eq:convergence_vr}
 \frac{1}{K} \sum_{k=0}^{K-1} \mathcal{D}_k \leq \frac{16 \tilde{F}\left(\bar{x}_0\right)}{K \gamma \tau} 
+  \frac{16\tilde{q}_4}{K \gamma \tau (1- \tilde{\delta} )} \|\widehat{\mathbf{d}}_{0}\|^2.
\end{equation}

By collecting all step-size conditions,  if the step-size $\gamma$ satisfies  $\gamma \leq \min \bar{\gamma}_{i=1, 7, 8, \ldots, 15}$,
then the states $\{\mathbf{X}_k\}$ generated by {\algsaga} converge to the stationary point, concluding the proof.


\bibliographystyle{IEEEtran} 
\bibliography{references}

\end{document}